\documentclass{ecai} 

\usepackage{amsmath,amsfonts,bm}

\def\eqref#1{equation~\ref{#1}}

\def\1{\bm{1}}

\DeclareMathAlphabet{\mathsfit}{\encodingdefault}{\sfdefault}{m}{sl}
\SetMathAlphabet{\mathsfit}{bold}{\encodingdefault}{\sfdefault}{bx}{n}

\newcommand{\R}{\mathbb{R}}

\newcommand{\KL}[2]{D_{\mathrm{KL}}(#1 \parallel #2)}

\usepackage{url}
\usepackage{latexsym}
\usepackage{amssymb}
\usepackage{amsmath}
\usepackage{amsthm}
\usepackage{booktabs}
\usepackage{enumitem}
\usepackage{graphicx}
\usepackage{color}
\usepackage{amsfonts}       
\usepackage{nicefrac}       
\usepackage{microtype}      
\usepackage{booktabs} 
\usepackage{tikz}
\usepackage{amsbsy}
\usepackage{amsgen}
\usepackage{mathrsfs}
\usepackage{amscd}
\usepackage{commath}
\usepackage[capitalize,noabbrev]{cleveref}
\usepackage[textsize=tiny]{todonotes}
\usepackage{algorithmic}
\usepackage{algorithm}
\usetikzlibrary{matrix,arrows,shapes,calc}
\usepackage{float}
\usepackage{pgfplots}
\tikzset{%
  every neuron/.style={
    circle,
    draw,
    minimum size=0.5cm
  },
  neuron missing/.style={
    draw=none, 
    scale=4,
    text height=0.333cm,
    execute at begin node=\color{black}$\vdots$
  },
}    
\usepackage{subcaption}

\theoremstyle{plain}
\newtheorem{theorem}{Theorem}[section]

\newtheorem{lemma}[theorem]{Lemma}
\newtheorem{corollary}[theorem]{Corollary}
\theoremstyle{definition}

\theoremstyle{remark}

\newtheorem{example}[theorem]{Example}

\newcommand{\Expect}{\mathbb{E}}

\newcommand{\BibTeX}{B\kern-.05em{\sc i\kern-.025em b}\kern-.08em\TeX}

\pgfplotsset{compat=1.18}
\begin{document}


\begin{frontmatter}


\paperid{2983} 


\title{A Neural Difference-of-Entropies Estimator for Mutual Information}


\author[A]{\fnms{Haoran}~\snm{Ni}\thanks{Corresponding Author. Email: haoran.ni.1@warwick.ac.uk}}
\author[B]{\fnms{Martin}~\snm{Lotz}}

\address[A]{CAMaCS, University of Warwick}
\address[B]{Mathematics Institute, University of Warwick}


\begin{abstract}
Estimating Mutual Information (MI), a key measure of dependence of random quantities without specific modeling assumptions, is a challenging problem in high dimensions. We propose a novel mutual information estimator based on parametrizing conditional densities using normalizing flows, a deep generative model that has gained popularity in recent years. This estimator leverages a block autoregressive structure to achieve improved bias-variance trade-offs on standard benchmark tasks.
\end{abstract}

\end{frontmatter}


\section{Introduction}\label{sec:intro}
Mutual Information (MI), a measure of dependence of random variables $X$ and $Y$, plays an important role in information theory~\citep{elements}, statistics and machine learning~\citep{tishby2000information,peng2005feature,vergara2014review,chen2018learning}, and biology and medical sciences~\citep{zhang2012inferring,sengupta2022survey}. For random variables $X$ and $Y$ with joint density $p$, the mutual information is defined as
\begin{equation*}
I(X;Y)=\KL{p}{p_X\otimes p_Y} = \mathbb{E}_{(X,Y)\sim p}\left[\log\frac{p(X,Y)}{p_X(X)p_Y(Y)}\right],
\end{equation*}  
where $p_X$ and $p_Y$ are the marginal densities of $X$ and $Y$, $p_X\otimes p_Y$ is the density of the product distribution, and $\KL{\cdot}{\cdot}$ is the Kullback-Leibler (KL) divergence. We consider the problem of estimating mutual information from a finite set of samples $\{(x_1,y_1),\dots,(x_N,y_N)\}$. 

Formally, an MI estimator $\hat{I}_N$ depends on independent and identically distributed (i.i.d.) random sample pairs $(X_1,Y_1),\dots,(X_N,Y_N)$, and should ideally be unbiased, consistent, and efficient. In addition, such an estimator should be effectively computable from large and high-dimensional data. 
We propose an unbiased and consistent mutual information estimator based on the difference-of-entropies (DoE) estimator suggested in~\citet{limit_mi}. This characterization expresses
the mutual information as the difference between the entropy of $X$ and the  conditional entropy of $X$ given $Y$,
\begin{equation*}
  I(X;Y) = H(X)-H(X\ | \ Y).
\end{equation*}
Each of the terms in this expression can be characterized as the infimum of a variational optimization problem.
Our implementation of this estimator is based on carefully chosen normalizing flows that simultaneously approximate the minimizing densities of each of the optimization problems.

\subsection{Overview of previous work}\label{sub:litreview} 
Traditional MI estimators are nonparametric estimators that depend on density estimation and Monte Carlo integration or on the calculation of $k$ nearest neighbors (kNN). Examples include the widely used KSG estimator by~\citet{ksg}, the nonparametric kNN estimator (kpN) by~\citet{kpn}, and improvements of the KSG estimator and a geometric kNN estimator by~\citet{biksg}.
These nonparametric methods are fast and accurate for low-dimensional and small-sized problems and are easy to implement. However, they suffer from the curse of dimensionality and do not scale well in machine learning problems since data sets can be relatively large and high-dimensional~\citep{paninski2003estimation}.

More recent parametric methods take advantage of deep learning architectures to approximate variational bounds on MI. These have been categorized into discriminative and generative approaches in \citet{smine}. Some state-of-the-art discriminative approaches include InfoNCE \citep{Infonce}, MINE \citep{mine}, SMILE \citep{smine}, CCMI \citep{ccmi} and DEMI \citep{demi}. \citet{Infonce} proposed a contrastive predictive coding (CPC) method that relies on maximizing a lower bound on mutual information.
The lower bound involves function approximators implemented with neural networks and is constrained by batch size $N$, leading to a method that is more biased, but with less variance. MINE, on the other hand, is based on the Donsker-Varadhan (DV) lower bound for the KL divergence. The fundamental limitations on approaches based on variational limits were studied by
~\citet{smine} and \citet{limit_mi}, the latter being the motivation for our approach.

Instead of constructing mutual information estimators based on variational lower bounds,~\citet{demi} proposed a classifier-based estimator called DEMI, 
where a parametrized classifier is trained to distinguish between the joint density $p(x,y)$ and the product $p(x)p(y)$. \citet{ccmi} proposed another classifier-based (conditional) MI estimator that is asymptotically equivalent to DEMI. However, it still relies on variational lower bounds and is prone to higher error than DEMI for finite samples, as summarized by \citet{demi}. 

Compared with discriminative approaches, generative approaches are less explored in MI estimation problems. A na\"ive approach using generative models for estimating MI is to learn the entropies $H(X), H(Y)$ and $H(X,Y)$ with three individual generative models, such as VAE~\citep{VAE} or Normalizing Flows~\citep{NF_currentmethods}, from samples. A method for estimating entropy using normalizing flows was introduced by~\citet{ao2022entropy}. 
Estimators based on the individual entropies will be highly biased and computationally expensive since the entropies are trained separately, while it is revealed that considering the enhancement of correlation between entropies in constructing MI estimators can improve the bias~\citep{biksg}. \citet{duong2023diffeomorphic} proposes the Diffeomorphic Information Neural Estimator (DINE), that takes advantage of the invariance of conditional mutual information under diffeomorphisms. 
An alternative approach to estimating MI using normalizing flows was recently proposed by~\citet{butakov2024mutual}. This approach takes advantage of the invariance of the point-wise mutual information under diffeomorphisms. In addition, the methods from~\citet{butakov2024mutual} allow for the estimation of mutual information using direct, closed-form expressions. Finally, we would like to point to the recently introduced MINDE estimator~\citep{minde}, which is based on diffusion models and represents a complementary approach.

From a practical point of view, the performance of MI estimators is often measured using standard data sets based on Gaussian distributions for which the ground truth is known. Recently, \citet{czyz2023beyond} proposed a collection of benchmarks to evaluate the performance of different MI estimators in more challenging environments.  

\subsection{Notation and conventions}\label{sub:notation}
The entropy of a random variable with density $p$ is defined as $H(X)=-\Expect[\log p(X)]$, and we sometimes write $H(p)$ to highlight the dependence on the density. Throughout this paper, we work with absolutely continuous random variables and distributions.


\section{Mutual Information and Normalizing Flows}
We begin by introducing the characterization of mutual information in terms of entropy that forms the basis of our approach. We then introduce normalizing flows and discuss an implementation of our mutual information estimator.

\subsection{Mutual Information and Entropy}
Given a pair of random variables $(X,Y)$, the conditional entropy of $X$ with respect to $Y$ is defined
as $H(X|Y) = H(X,Y)-H(Y)$,
where $H(X,Y)$ is the joint entropy of $(X,Y)$ (not to be confused with the cross-entropy, introduced below).
The mutual information can be expressed in terms of entropies via the 3H principle: 
\begin{align}\label{eq:3h}
\begin{split}
I(X;Y) &= H(X)+H(Y)-H(X,Y) \\
&\stackrel{(*)}{=} H(X)-H(X|Y).
\end{split}
\end{align}
The characterization (*) is the basis of the difference-of-entropies (DoE) estimator introduced by~\citet{limit_mi}.

The entropy of a random variable can be characterized as the solution of a variational optimization problem involving the cross-entropy. The cross-entropy between random variables $X$ and $Y$ with densities $p$ and $q$, respectively, is defined as
\begin{equation*}
  Q(p,q) := -\Expect_{p}[\log q(X)].
\end{equation*}
One easily checks that the cross-entropy, entropy, and KL divergence are related via
\begin{equation}\label{eq:cross-kl}
  Q(p,q) = H(X)+\KL{p}{q}.
\end{equation} 
The KL divergence is non-negative and satisfies $\KL{p}{q}=0$ if and only if $p=q$ almost everywhere. A well-known consequence of this fact is the following characterization of the entropy of a random variable with density $p$:
\begin{equation*}
  H(X) = \inf_{q} Q(p,q),
\end{equation*}
where the infimum is taken over all probability densities $q$.

The conditional entropy $H(X|Y)$ is itself the cross-entropy of the conditional density $p_{X|Y}=p/p_Y$ with respect to the joint density $p$,
\begin{equation*}
  H(X|Y) = -\Expect_p\left[\log\frac{p(X,Y)}{p_Y(Y)}\right] = Q(p,p_{X|Y}).
\end{equation*}
Note that a conditional probability density is not a joint density, as it does not integrate to $1$, but the definition of cross-entropy and KL-divergence still makes sense. The proof of the following result is simple and is included for reference.

\begin{lemma}\label{le:lemmainf}
Let $(X,Y)$ be a pair of random variables with joint density $p$. Then
\begin{equation*}
  H(X|Y) = \inf_{q} Q(p,q),
\end{equation*}
where the infimum is over all conditional densities, i.e., non-negative functions $q(x|y)$ such that $\int_x q(x|y) \ \mathrm{d}x=1$ for all $y$. 
\end{lemma}

\begin{proof}
Let $q(x|y)$ be a conditional density. Then
\begin{align*}
  Q(p,q) &= -\Expect_p[\log q(X|Y)]\\
  &= H(X|Y) +\Expect_p[\log p(X,Y)]-\Expect_p[\log(q(X|Y)p_Y(Y))]\\
  &= H(X|Y)-H(X,Y)+Q(p,\tilde{q})\\
  &= H(X|Y)+\KL{p}{\tilde{q}},
\end{align*}
where $\tilde{q}(x,y)=q(x|y)\cdot p_Y(y)$ is a probability density. By~\eqref{eq:cross-kl}, $Q(p,\tilde{q})\geq H(X,Y)$, with equality if and only if $\tilde{q}=p$, i.e., $q=p_{X|Y}$. 
\end{proof}

As a consequence of the proof (or by direct inspection) we get the following observation.

\begin{corollary}
Let $(X,Y)$ be a pair of random variables with density $p$ and let $q(x,y)$ be a probability density. Then
\begin{equation*}
  Q(p,q/p_Y) = \KL{p}{q}+H(X|Y)
\end{equation*}
\end{corollary}

Together with the 3H principle,~\eqref{eq:3h}, we get

\begin{equation}\label{eq:mainchar}
  I(X;Y) = \inf_{q_X}Q(p_X,q_X)-\inf_{q_{X|Y}} Q(p,q_{X|Y}).
\end{equation}

Given data $\{(x_i,y_i)\}_{i=1}^N$, the resulting difference-of-entropies (DoE) estimator, as suggested by~\citet{limit_mi}, consists of minimizing the objectives

\begin{equation}\label{eq:cross-entropy-est}
\begin{split}
  \hat{Q}(p_X,q_X)& = -\frac{1}{N} \sum_{i=1}^N \log q_X(x_i),\\
  \hat{Q}(p,q_{X|Y})& = -\frac{1}{N}\sum_{i=1}^N \log q_{X|Y}(x_i|y_i)
\end{split}
\end{equation}

with respect to $q_X$ and $q_{X|Y}$. 
In our implementation of the DoE estimator, we parametrize these densities jointly, rather than separately, using block autoregressive normalizing flows.

\subsection{Normalizing flows}\label{sec2.2_nf&kl_divergence}
A popular way of estimating densities is via normalizing flows, where the density to be estimated is seen as the density of a push-forward distribution of a simple base distribution, and the transformation is implemented using invertible neural networks. 
Let $g\colon \R^n\to \R^n$ be a measurable function, and let $\mu$ be a probability measure. The push-forward measure $g_*\mu$ is defined as
\begin{equation*}
    g_*\mu(A)=\mu(g^{-1}(A))
\end{equation*}
for all measurable $A$. 
%
The density of a random variable $X$ that has the push-forward distribution of an absolutely continuous random variable $Z$ with density $p_{Z}$ with respect to a diffeomorphism $g$ is also absolutely continuous, with a density function $p_{X}$ given by
\begin{align*}
\begin{split}
    p_{X}(x)=p_{Z}(g^{-1}(x))\cdot\abs{\det\textnormal{d}g(g^{-1}(x))}^{-1},
\end{split}
\end{align*}
where $\textnormal{d}g(z)$ denotes the differential of $g$ at $z$ (in coordinates, given by the Jacobian matrix). 

It is known that any continuous distribution with density $p_{X}$ satisfying some mild conditions can be generated from the uniform distribution on a cube $[0,1]^n$ (and hence, by invertibility, from any other distribution satisfying the same conditions) if the transformation $f$ can have arbitrary complexity~\citep{triangular}. However, as is common with universal approximation results, this result does not translate into a practical recipe~\citep{Liu2024Universal}. A more practical approach is to use a composition of simple functions implemented by neural networks, which have sufficient expressive power.
An obvious but important property of diffeomorphisms is that they are composable. Specifically, let $g_1,g_2, \dots, g_K$ be a set of $K$ diffeomorphisms and denote by $g=g_K\circ g_{K-1}\circ\cdots\circ g_1$ the composition of these functions. The determinant of the Jacobian is then given by
\begin{equation*}
    \det\textnormal{d}g(z)=\prod_{i=1}^{K}\det\textnormal{d}g_i(z_i),
\end{equation*}
where $z_i=g_{i-1}\circ \cdots \circ g_1(z)$ for $i\geq 2$ and $z_1=z$ and $z_{K+1}=x=g(z)$.
Similarly, for the inverse of $f$, we have 
\begin{equation*}
    g^{-1}=g^{-1}_1\circ\cdots\circ g^{-1}_K,
\end{equation*}
and the determinant of the Jacobian is computed accordingly. Thus, we can construct more complicated functions with a set of simpler, bijective functions. The two crucial assumptions in the theory of normalizing flows are thus invertibility ($g^{-1}$ should exist) and simplicity (each of the $g_i$ should be simple in some sense). The inverse direction, $f=g^{-1}$, is called the normalizing direction: it transforms a complicated distribution into a Gaussian, or normal distribution. For completeness and reference, we reiterate the transformation rule in terms of the normalizing map:

\begin{equation}\label{eq:normalizing-det}
  \log p_X(x) = \log p_Z(f(x)) + \log |\det \mathrm{d}f(x)|.
\end{equation}

Normalizing flows are fitted by minimizing the KL divergence between a model $p_{X}(x;\Theta)$ and an unknown target distribution $p_{X}^*(x)$ from which we only see samples. Here, the model parameters are denoted as $\Theta=\{\phi,\psi\}$, where $\phi$ are the parameters of the normalizing function $f_{\phi}$, and $\psi$ are the parameters of the base density $p_{Z}(z;\psi)$. Because the KL divergence is asymmetric, the order in which the probabilities are listed is important, which leads to two different cost functions, the forward and the reverse KL divergence. In our work, we only focus on the forward KL divergence $D_{\textnormal{KL}}\ (p_{X}^*\parallel p_{X}(\cdot ;\Theta))$ since it applies in situations when we have no way to evaluate the target density $p_{X}^*(x)$, but we have (or can generate) samples from the target distribution. 

In light of~\eqref{eq:cross-kl}, minimizing the forward KL divergence is equivalent to minimizing the cross-entropy:
\begin{align}
\begin{split}\label{crossentropy}
   \mathcal{L}(\Theta) &:=Q(p_{X}^*,p_{X}(\cdot ;\Theta)) =-\mathbb{E}_{p_{X}^*}[\log p_{X}(X;\Theta)]\\
    &=-\mathbb{E}_{p_{X}^*}[\log p_{Z}(f_\phi(X);\psi)+\log|\det \textnormal{d}f_\phi(X)| ].
\end{split}
\end{align}

Given a set of samples $\{x_j\}_{j=1}^{N}$ from $p_{X}^*(x)$, $\mathcal{L}(\Theta)$ can be estimated by replacing the expectation with the empirical mean, which leads to the cost function
\begin{equation}\label{KL_monte_carlo}
   \hat{\mathcal{L}}(\Theta) := -\frac{1}{N}\sum_{j=1}^{N}(\log p_{Z}(f_\phi(x_j);\psi)+\log|\det \textnormal{d}f_\phi(x_j)|).
\end{equation}
Equation~\ref{KL_monte_carlo} is a Monte Carlo estimate of the cross entropy between the target distribution and the model distribution. The cost function $\mathcal{L}(\Theta)$ is minimized when $p_X^* = p_X(\cdot ;\Theta)$, and the optimal value is the entropy of $X$. If the model is expressive enough to characterize the target distribution, then minimizing~\ref{KL_monte_carlo}
over the parameters yields an entropy estimator. 

\subsubsection{Block autoregressive flows}\label{sec:block}
Autoregressive flows~\citep{iaf} are normalizing flows with the convenient property that their Jacobian is triangular. 
Block neural autoregressive flows (B-NAF), introduced by~\citet{BNAF}, are flows that are autoregressive and monotone, but that are implemented using a single neural network architecture, rather than relying on conditioner networks. More specifically, a block autoregressive flow is given as a sequence of transformations
\begin{equation*}
 f\colon \R^d\to \R^{da_1}\to \cdots \to \R^{da_\ell}\to \R^{d},
\end{equation*}
where each $f^{k}\colon \R^{da_{k}}\to \R^{da_{k+1}}$ is given by $f^{k}(x)=\sigma(W^{k}x+b^{k})$ with $\sigma$ a strictly increasing activation function, and $W^{k}$ is a block matrix of the form
\begin{equation*}
  W^{k} = \begin{bmatrix}
    g(B_{11}^{(k)}) & 0 & \cdots & 0\\
    B_{21}^{(k)} & g(B_{22}^{(k)}) & \cdots & 0\\
    \vdots & \vdots & \ddots & \vdots\\
    B_{d1}^{(k)} & B_{d2}^{(k)} & \cdots & g(B_{dd}^{(k)})
\end{bmatrix},
\end{equation*}
where each $B_{ij}^{(k)}\in R^{a_{k+1}\times a_k}$ and $g(x)=\exp(x)$ is applied component-wise, to ensure that the entries are positive. We set $a_0=a_{\ell+1}=1$. It is not hard to see that the $i$-component of $f(x)$ only depends on $x_1,\dots,x_i$.
Since the product of block diagonal matrices with blocks of size $a\times b$ and $b\times c$, respectively, is block diagonal with size $a\times c$, the composition $f$ has a lower triangular Jacobian with positive diagonal entries, and hence is invertible. The determinant of the triangular Jacobian matrix is the product of the diagonal entries $\partial f_i/\partial x_i$, each of which can be computed as product
\begin{equation*}
  \frac{\partial f_i}{\partial x_i} = \prod_{k=0}^{\ell} g(B_{ii}^{(k)}).
\end{equation*}
In practice, implementations of B-NAF use masked networks and gated residual connections to improve stability, but this does not alter the analysis. Just as with neural autoregressive flows, it can be shown that B-NAF are universal density estimators.

\section{Joint estimation of Mutual Information}\label{autoreg_mi_estimator}
Our goal is to minimize the functions in~\eqref{eq:mainchar}, where the density $q_X(x)$ and the conditional density $q_{X|Y}(x|y)$ are parametrized using normalizing flows. 
We implement the difference of entropies (DoE) estimator by constructing a specific neural network structure that can estimate the two entropies in~\eqref{eq:cross-entropy-est} in the same framework by ``deactivating'' the certain sub-network. Technically, this is implemented by using a mask to set the contributions coming from one part of the network to another to zero.

To motivate the architecture, consider the network in Figure~\ref{fig:nnn}, implementing a flow $f\colon \R^2\to \R^2$ given as a composition $f=f^2\circ f^1\circ f^0$ with $f^0\colon \R^2\to \R^4$, $f^1\colon \R^4\to \R^4$ and $f^2\colon \R^4\to \R^2$. Hence, $a_1=a_2=2$ and the corresponding neural network has the form shown in Figure~\ref{fig:nnn}.
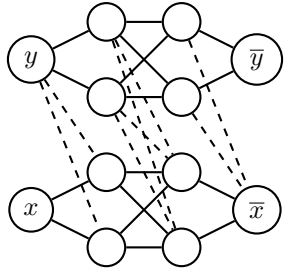
\begin{figure}[!ht]
\centering
\begin{tikzpicture}[thick, scale=1.0]
  \node [every neuron] (I1) at (-2,1) {$y$};
  \node [every neuron] (I2) at (-2,-1) {$x$};
  \node [every neuron] (M1) at (-1,1.5) {};
  \node [every neuron] (M2) at (-1,0.5) {};
  \node [every neuron] (M3) at (-1,-0.5) {};
  \node [every neuron] (M4) at (-1,-1.5) {};
  \node [every neuron] (N1) at (0,1.5) {};
  \node [every neuron] (N2) at (0,0.5) {};
  \node [every neuron] (N3) at (0,-0.5) {};
  \node [every neuron] (N4) at (0,-1.5) {};
  \node [every neuron] (O1) at (1,1) {$\overline{y}$};
  \node [every neuron] (O2) at (1,-1) {$\overline{x}$};
  
  \draw[-] (I1) -- (M1);
  \draw[-] (I1) -- (M2);
  \draw[-, dashed] (I1) -- (M3);
  \draw[-, dashed] (I1) -- (M4);
  \draw[-] (I2) -- (M3);
  \draw[-] (I2) -- (M4);
  \draw[-] (M1) -- (N1);
  \draw[-] (M1) -- (N2);
  \draw[-] (M2) -- (N1);
  \draw[-] (M2) -- (N2);
  \draw[-, dashed] (M1) -- (N3);
  \draw[-, dashed] (M1) -- (N4);
  \draw[-] (M3) -- (N3);
  \draw[-] (M3) -- (N4);
  \draw[-, dashed] (M2) -- (N3);
  \draw[-, dashed] (M2) -- (N4);
  \draw[-] (M4) -- (N3);
  \draw[-] (M4) -- (N4);
  \draw[-] (N1) -- (O1);
  \draw[-] (N2) -- (O1);
  \draw[-, dashed] (N1) -- (O2);
  \draw[-, dashed] (N2) -- (O2);
  \draw[-] (N3) -- (O2);
  \draw[-] (N4) -- (O2);

\end{tikzpicture}
\hspace{5mm}
\caption{A Block Autoregressive Flow $f(y,x)$. Solid lines represent positive weights.}\label{fig:nnn}
\hspace{150mm}
\end{figure}

Recall that block autoregressive flows have the property that $f_i$ depends only on the first $i$ variables. In particular, we can express the function $f$ as
\begin{equation*}
 f(y,x) = (f_1(y), f_2(y,x)).
\end{equation*}
The Jacobian determinant is the product of the partial derivatives $\partial f_1/\partial y$ and $\partial f_2/\partial x$ (see Section~\ref{sec:block}). Suppose $p(x,y)$ is a standard Gaussian density, so that $p(x,y)=p_X(x)p_Y(y)$, and that we have data $(x_i,y_i)$ from an unknown distribution $q$. The cost function~\eqref{KL_monte_carlo} for learning a normalizing flow takes the form
\begin{equation}\label{eq:cross-example}
\begin{split}
   \hat{\mathcal{L}}(\Theta) =& -\frac{1}{N}\sum_{i=1}^N \left(\log p_X(f_2(y_i,x_i))+\log \frac{\partial f_2}{\partial x}(y_i,x_i)\right)\\
   &+\left(\log p_Y(f_1(y_i)) +\log \frac{\partial f_1}{\partial y}(y_i)\right).
\end{split}
\end{equation}
The components $f_1$ and $f_2$ depend on a distinct set of weights in the neural network. Optimizing only the part of~\eqref{eq:cross-example} involving $f_1$ on data $\{y_i\}$ gives an estimate for the entropy of $Y$, while optimizing the part with $f_2$ on data $\{(x_i,y_i)\}$ gives rise to an estimate of the conditional entropy $H(X\ | \ Y)$. Moreover, if we deactivate the weights in off-diagonal blocks (the dashed lines), then optimizing this part on data $\{x_i\}$ gives an estimate of $H(X)$. Note that training for $H(X\ | \ Y)$ and setting the off-diagonal weights to zero does not automatically give an estimator for $H(X)$. It is, however, conceivable that one can begin with a network that approximates $H(X)$ and then optimize the off-diagonal weights to obtain an approximation of $H(X|Y)$.

In general, we consider a flow $f\colon \R^{2n}\to \R^{2n}$ with a block autoregressive structure, given by $f(y,x) = (f_1(y),f_2(y,x))$ with $x\in \R^n$, $y\in \R^n$.
The function $f_2$ is a composition of layers of the form
\begin{equation*}
  \sigma(W_{21}^{(\ell)}y^{(\ell-1)}+W_{22}^{(\ell)}x^{(\ell-1)}+b^{(\ell)}),
\end{equation*}
where $(y^{(\ell-1)},x^{(\ell-1)})$ is the output of the previous layer of the flow $f$. Consider the cost function
\begin{equation*}
  \mathcal{L}_1 = -\frac{1}{N}\sum_{i=1}^N \left(\log p(f_2(y_i,x_i)) + \log \det |\mathrm{d}_xf_2(y_i,x_i)|\right),
\end{equation*}
where now we simply write $p$ for the density of a Gaussian.
Optimizing this function gives an estimate of the conditional entropy $H(X\ | \ Y)$. If, on the other hand, we set the off-diagonal weights to zero and optimize the resulting function $\mathcal{L}_2$, we get an estimator for the entropy $H(X)$. This motivates Algorithm~\ref{alg:mi_auto}, which optimizes for $H(X\ | \ Y)$ and $H(X)$ simultaneously.

\begin{algorithm}[tb]
\caption{Normalizing Flows MI Estimation}\label{alg:mi_auto}
\begin{algorithmic}
\STATE {\bfseries Input: } data $(x_i,y_i)$
\STATE Initialize model parameters $\phi$.
\REPEAT
\STATE Draw minibatch $S$ of $M$ samples $\{(x_i,y_i)\}$
\STATE Evaluate:
    \begin{align*}
    \mathcal{L}_1 & = -\frac{1}{M}\sum_{(x,y)\in S}(\log p(f_2(y,x; \phi))+\log|\det \textnormal{d}_xf_2(y,x; \phi)|)
    \end{align*}
\STATE Update the parameters by gradients: $\phi=\phi-\nabla\mathcal{L}_1$
\STATE Deactivate the off-diagonal weights, call new parameters $\phi'$
\STATE Evaluate:
    \begin{align*}
    \mathcal{L}_2 & = -\frac{1}{M}\sum_{(x,y)\in S} (\log p(f_2(y,x; \phi'))+\log|\det \textnormal{d}_xf_2(y,x; \phi')|)
    \end{align*}
\STATE Update the parameters by gradients: $\phi'=\phi'-\nabla\mathcal{L}_2$
\UNTIL{Convergence}
\STATE {\bfseries Output: } $\hat{I}(X,Y)=\mathcal{L}_2-\mathcal{L}_1$
\end{algorithmic}
\end{algorithm}

Algorithm~\ref{alg:mi_auto} can be generalized to any normalizing flows with inner autoregressive structure between $X$ and $Y$. Compared with general autoregressive flows which usually model the autoregressive functions as conditioner neural networks, BNAF has not only the superior expressive power, but also the easy computation of the Jacobian matrix and the straightforward deactivation operation given by the block-wise matrix form of autoregressive functions. The theoretical justification based on universal approximation results for Block Autoregressive Flows is provided in Section B of the supplementary materials~\citep{ni2025_arxiv}.

\begin{figure*}[ht]
\centering
    \begin{subfigure}[b]{1.0\linewidth}
    \centering
    \includegraphics[scale=0.60]{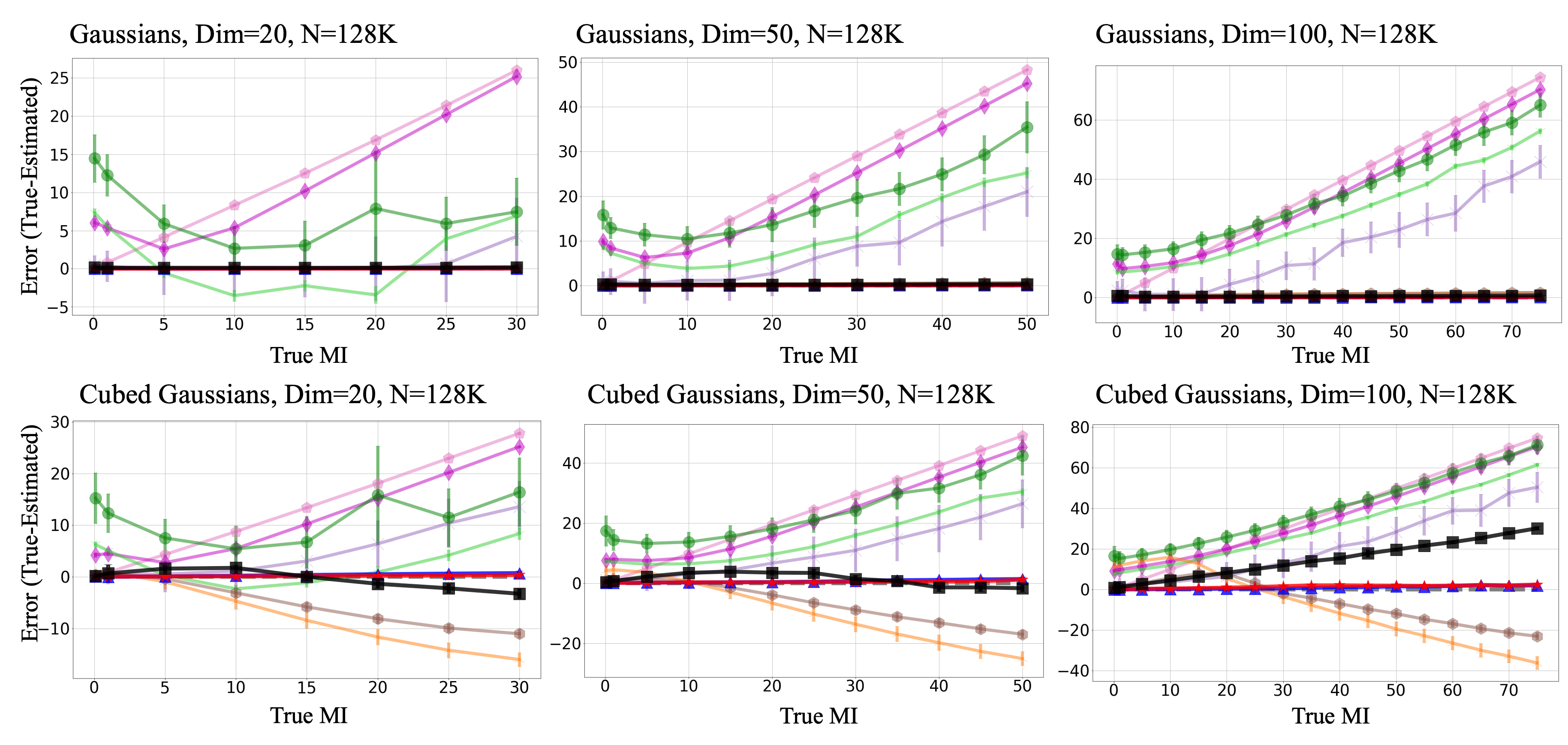}
    \end{subfigure}\\

    \begin{subfigure}[b]{1.0\linewidth}
    \centering
    \includegraphics[scale=0.65]{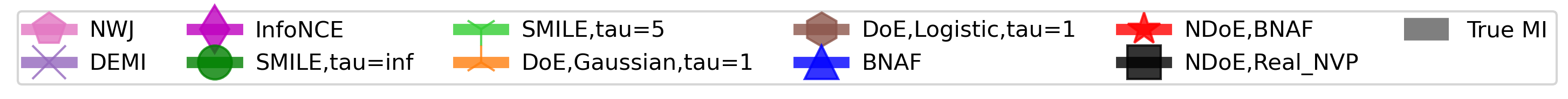}
    \end{subfigure}
    \hspace{5mm}
\caption{MI estimation between multivariate Gaussian variables (Top) and between multivariate Gaussian variables with a cubic transformation (Bottom). The size of training data are 128K. The estimation error $(I(x,y)-\hat{I}(x,y))$ are reported. Closer to zero is better.}

\label{Gaussian_N128K}
\end{figure*}   

\begin{figure*}[ht]
\centering
   \begin{subfigure}[b]{1.0\linewidth}
   \centering
    \includegraphics[scale=0.60]{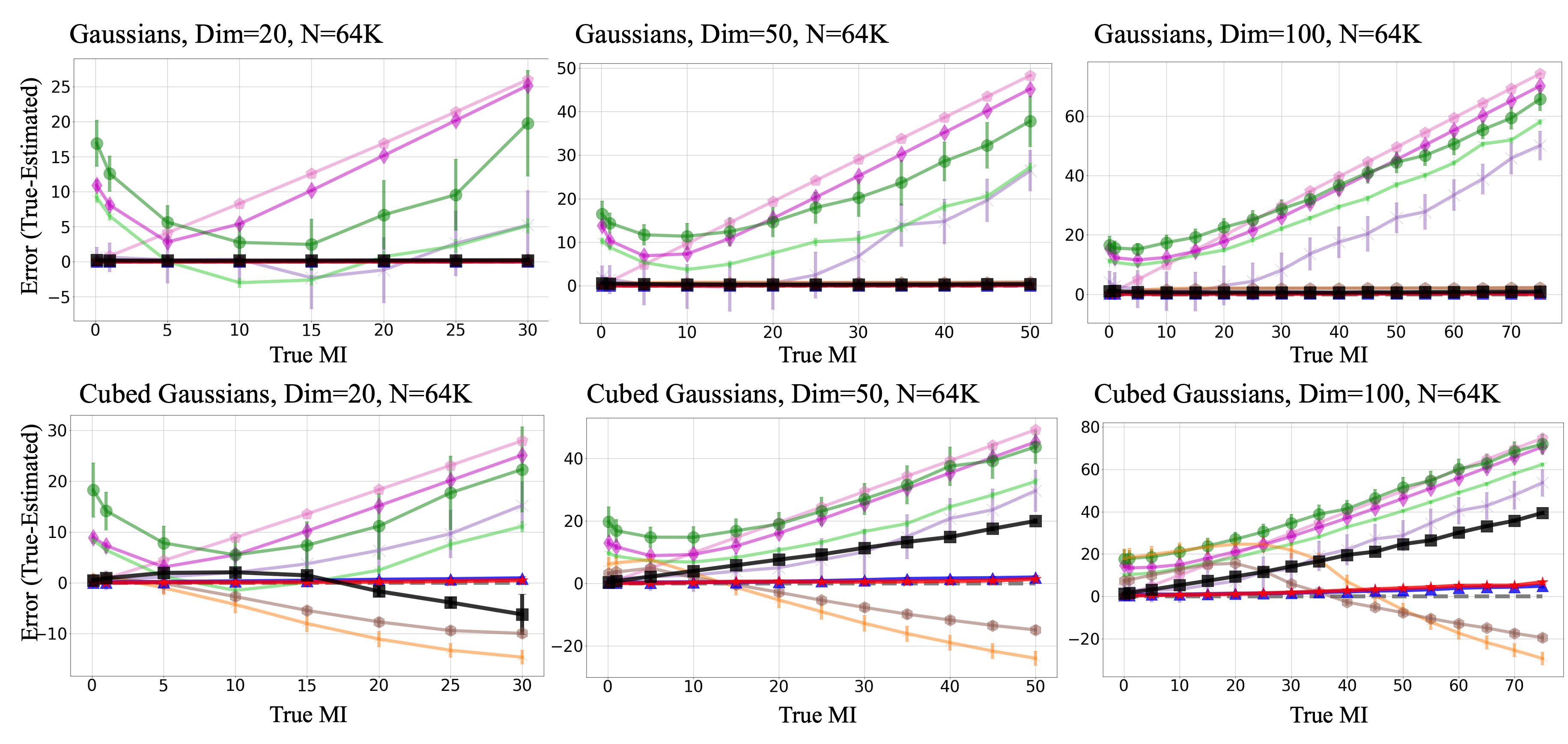}
    \end{subfigure}\\

    \begin{subfigure}[b]{1.0\linewidth}
    \centering
    \includegraphics[scale=0.65]{plots/legend_Gaussian.png}
    \end{subfigure}
    \hspace{5mm}
\caption{MI estimation between multivariate Gaussian variables (Top) and between multivariate Gaussian variables with a cubic transformation (Bottom). The size of training data are 64K. The estimation error $(I(x,y)-\hat{I}(x,y))$ are reported. Closer to zero is better.}

\label{Gaussian_N64K}
\end{figure*} 

\begin{figure*}[ht]
\centering
    \begin{subfigure}[b]{1.0\linewidth}
    \centering
    \includegraphics[scale=0.60]{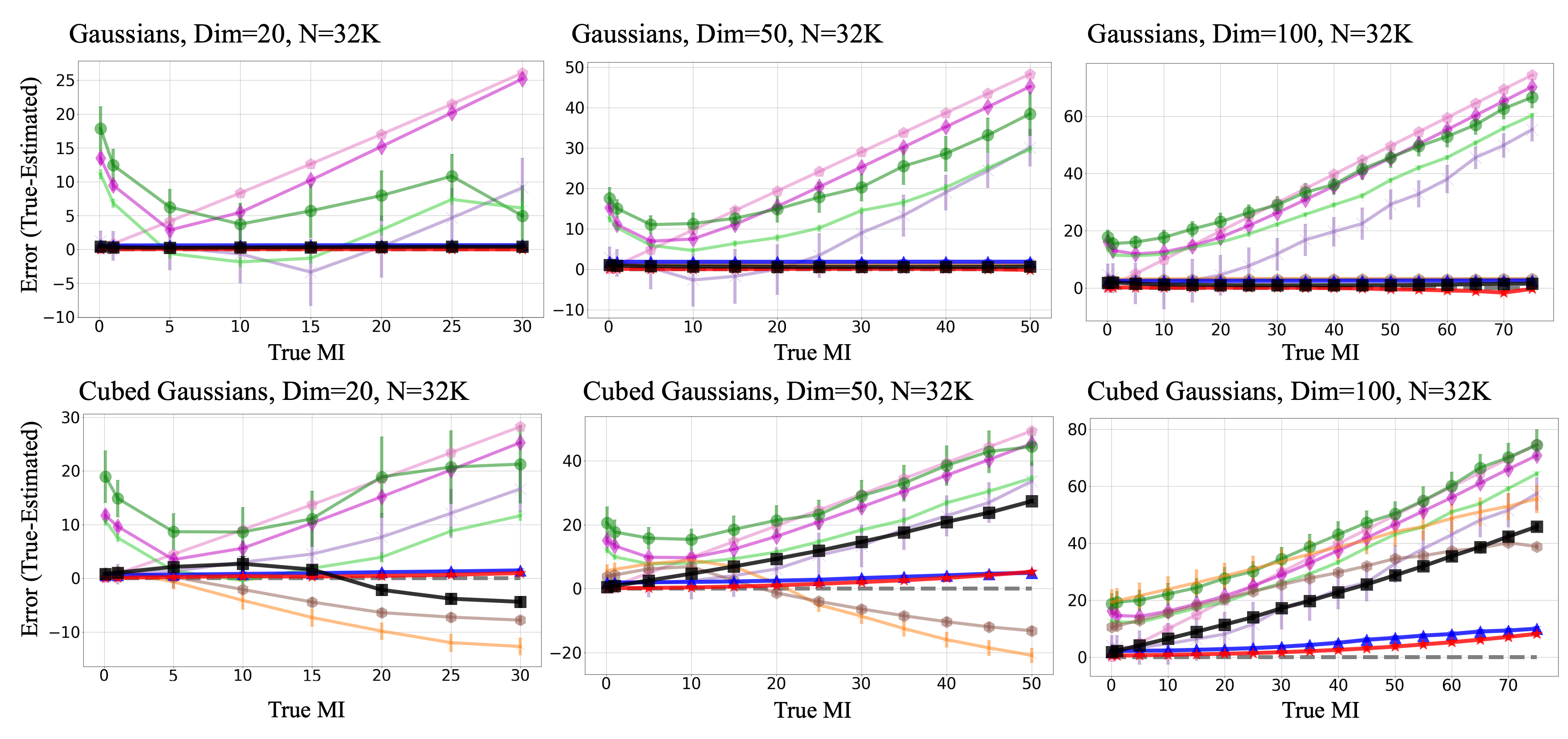}
    \end{subfigure}\\
    \begin{subfigure}[b]{1.0\linewidth}
    \centering
    \includegraphics[scale=0.65]{plots/legend_Gaussian.png}
    \end{subfigure}
    \hspace{5mm}
\caption{MI estimation between multivariate Gaussian variables (Top) and between multivariate Gaussian variables with a cubic transformation (Bottom). The size of training data are 32K. The estimation error $(I(x,y)-\hat{I}(x,y))$ are reported. Closer to zero is better.}

\label{Gaussian_N32K}
\end{figure*} 

\begin{figure*}[ht]
\centering
    \begin{subfigure}[b]{1.0\linewidth}
    \centering
    \includegraphics[scale=0.60]{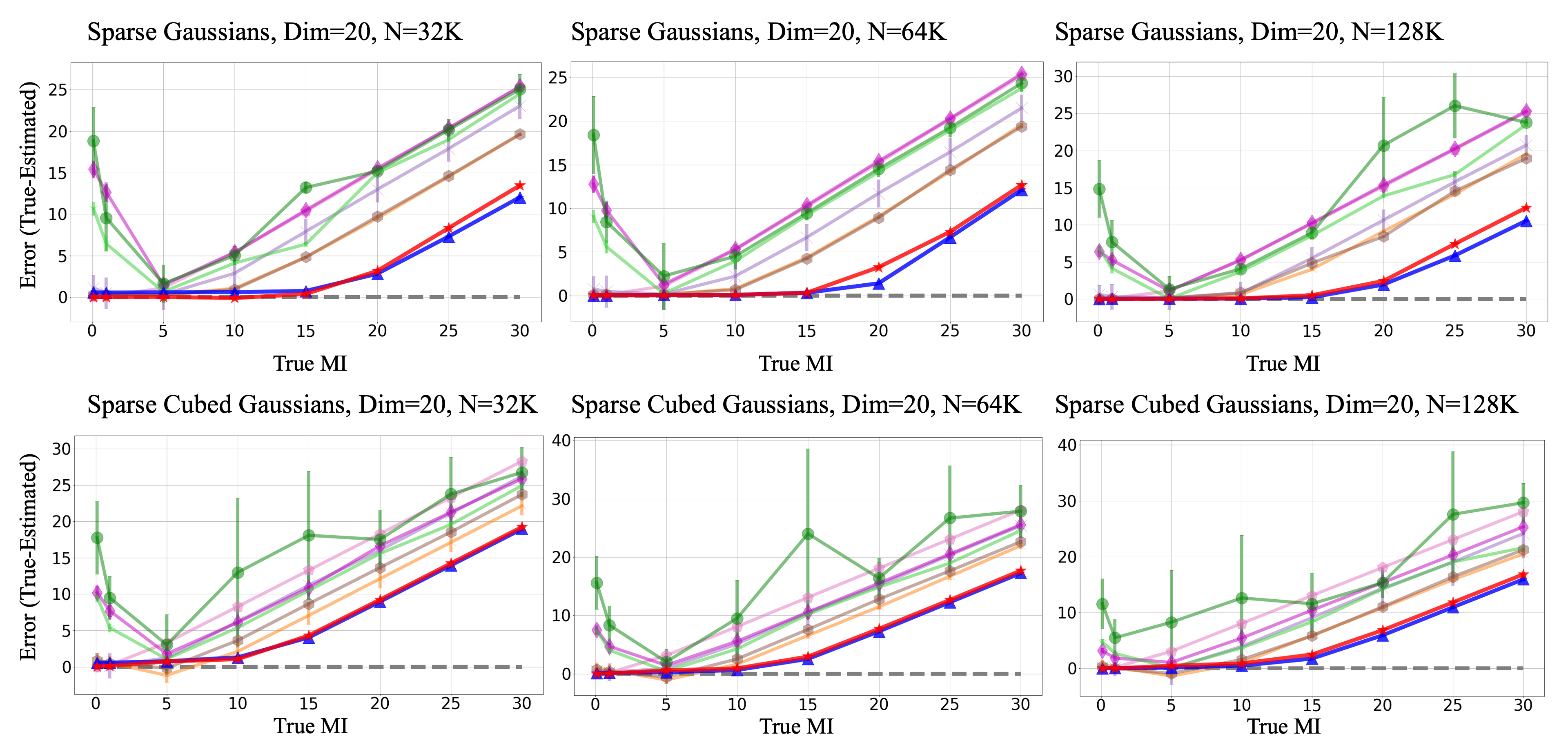}
    \end{subfigure}\\
    \begin{subfigure}[b]{1.0\linewidth}
    \centering
    \includegraphics[scale=0.65]{plots/legend_Gaussian.png}
    \end{subfigure}
    \hspace{5mm}
\caption{MI estimation between multivariate Sparse Gaussian variables (Top) and between multivariate Sparse Gaussian variables with a cubic transformation (Bottom). The size of training data are 128K. The estimation error $(I(x,y)-\hat{I}(x,y))$ are reported. Closer to zero is better.}
\label{Sparse_Gaussians}
\end{figure*}

\section{Numerical Experiments}\label{numerical}
We implemented several experimental settings from prior work \citep{mine,smine,variationalbounds,learningdeepreps,czyz2023beyond} to evaluate the performance of the proposed estimator. We first focus on the accuracy of the resulting estimates on synthetic Gaussian examples, where the true value of MI can be calculated analytically. In the Section C of supplementary materials~\citep{ni2025_arxiv}, we report on additional experiments on extremely small-sized datasets and non-Gaussian distributions. The final experiments are the long-run training behavior on the proposed estimator. All experiments were conducted on a computing cluster using Nvidia Quadro RTX 6000 GPUs. The implementation is available at: https://github.warwick.ac.uk/u1774790/nfmi.

\subsection{MI Estimation on Correlated Multivariate Gaussians}\label{4.1.0}
In this experiment, we sampled from two correlated Gaussian random variables $X$ and $Y$, for which the MI can be exactly obtained from their known correlation. The different MI estimators were trained on datasets with varying dimensionality of $X,Y$ ($20$-d, $50$-d and $100$-d), sample size ($32$K, $64$K and $128$K) and true MI to characterize the relative behaviors of every MI estimator. Additionally, we conduct an experiment by applying an element-wise cubic transformation on $y_i\rightarrow y_i^3$. This generates the non-linear dependencies in data without changing the ground truth of MI. The performance of trained estimators are evaluated on a different testing set of $10240$ samples. Czy{\.z} et al.~\citep{czyz2023beyond} mentioned that Gaussians with sparse interactions between $X$ and $Y$ could be a challenging benchmark for MI estimations. We then sample from Gaussians with $\textnormal{Cor}(X1,Y1), \textnormal{Cor}(X2,Y2) >0$ and there is no correlation between any other (distinct) variables. We named it Sparse Gaussian as the covariance matrix $Cov(X,Y)$ is a sparse matrix in this case. We assessed our methods along with the following baselines: 1. \textbf{DEMI}~\citep{demi}, with the parameter $\alpha=0.5$. 2. \textbf{SMILE}~\citep{smine}, with three clipping parameters $\tau\in\{5.0, \infty\}$. For $\tau=\infty$, it is equivalent to the MINE~\citep{mine}; 3. \textbf{InfoNCE}~\citep{Infonce}, which is the method in contrastive predictive coding (CPC); 4. \textbf{NWJ}~\citep{NWJ}, which is the method based on estimating the likelihood ratios by convex risk minimization;  5. \textbf{DoE}~\citep{limit_mi}, the DoE method, where the distributions is parameterized by isotropic Gaussian (correct) or logistic (misspecified), with three parameters $\tau=1.0$ that clips the gradient norm in training; 6. \textbf{BNAF}, approximating the entropies respectively in the MI using two separate Block Neural Autoregressive Flows; 7. \textbf{NDoE, BNAF}, the proposed method with BNAF structure; 8. \textbf{NDoE, Real NVP}, the proposed method with Real NVP structure. The implementation example is provided in Section A of supplementary materials~\citep{ni2025_arxiv}. We noticed that the comparison between our method, as a generative model, and other discriminative methods can be difficult since the neural network structure and the model parametrizations are different. To make the comparison as fair as possible, we used the same neural network architecture for all discriminative methods, which is a multi-layer perceptron with an initial concatenation layer, two fully connected layers with $512$ hidden units for each layer and ReLU activations, and a linear layer with a single output. In terms of our proposed method, we constructed the flow with 2 BNAF transformation layers and tanh activations, and a linear BNAF layer to reset the dimensionality. The BNAF layers use $20\times20$-d, $10\times50$-d, $6\times100$-d hidden dimensions for $20$-d, $50$-d and $100$-d data respectively, which is roughly the same as the 512 hidden units in discriminative methods. For Real NVP layers, we let each of the scale and translation functions be two layers multi-layer perceptron with $128$ hidden units for each layer and ReLU activations. Each MI estimator was trained for $50$ epochs with a mini-batch of $128$. Due to the vanishing and exploding gradient issues in Real NVP, we applied the Adamax optimizer with a fine-tuned learning rate. For all other optimizations, the Adam optimizer with a learning rate of 0.0005 was used. All results were computed over 10 runs on the testing sets generated with different random seeds to ensure robustness and generalizability.\\

\textbf{Results.} The results for a sample size of 128K are shown in Figure~\ref{Gaussian_N128K}, while the results for sample sizes of 64K and 32K are presented in Figure~\ref{Gaussian_N64K} and Figure~\ref{Gaussian_N32K}, respectively. The Sparse Gaussian results for the 20-dimensional case are plotted in Figure~\ref{Sparse_Gaussians}. Overall, all the discriminative methods tend to underestimate MI. This issue does not occur in our proposed flow-based models for Gaussian variables, likely due to the fact that the base distribution is itself Gaussian. For the cubic Gaussian case, the underestimation is much milder compared to other methods, though the underestimating bias increases as the true MI becomes larger. 

Among all the methods, our proposed model achieved better performance across different dimensionalities and sample sizes. While \textbf{DoE} methods performed well for Gaussian variables, they exhibited a large bias when applied to cubic Gaussians. When comparing \textbf{NDoE, BNAF} with \textbf{BNAF}, we observed that for smaller sample sizes (or insufficient training steps), \textbf{BNAF} exhibits a slight bias across all true MI values. Additionally, for cubic cases, \textbf{BNAF} shows a larger bias when MI is close to zero, an issue not observed with \textbf{NDoE, BNAF}. The work by \citet{smine} attributed this as a shortcoming of generative models, but we believe our proposed method mitigates this issue, as the bias in entropy estimation vanishes by approximating entropies using the same neural network.

In the cubic cases, \textbf{NDoE, Real NVP} exhibits a larger bias, though it still outperforms discriminative methods, particularly when the sample size is sufficiently large. In the 20-dimensional Gaussian case, \textbf{SMILE} occasionally overestimated MI, which we will further analyze in the long-run training experiments. For the Sparse Gaussian case, both \textbf{NDoE, BNAF} showed small biases when the true MI is small, and \textbf{BNAF} outperformed \textbf{NDoE, BNAF} for larger MI. Both methods consistently outperformed other approaches across different sample sizes. However, \textbf{NDoE, Real NVP} failed to achieve realistic results in the Sparse Gaussian case. This may be due to the Real NVP architecture capturing fewer intrinsic dependencies between $X$ and $Y$.

\section{Conclusions}\label{sec5*_Conclusions}
In this research, we proposed a new MI estimator which is based on the block autoregressive flow structure and the difference-of-entropies (DoE) estimator. Theoretically, our method converges to true mutual information as the number of samples increases and with large-enough neural network capacity.
The accuracy of our estimator then depends on the ability of the block autoregressive flows in predicting the true posterior probability of items in the test set. A theoretical analysis is provided in Section B of supplementary materials~\citep{ni2025_arxiv}.
We discussed the connections and differences between our approach and other approaches, including the lower bound approaches of MINE and SMILE, InfoNCE (CPC), and the classifier-based approaches of CCMI and DEMI. We also demonstrate empirical advantages of our approach over the state-of-the-art methods for estimating MI in synthetic data. Given its simplicity and promising performance, we believe that our method is a good candidate for use in research that optimizes MI. In future work, we aim to expand our experiments to additional data, including the image-like data of~\citet{butakov2024mutual} and the recently published benchmarks in~\citet{lee2024benchmark}.

Despite its promising results, the proposed method has limitations in its current form. As a method that depends on the particular neural network architecture used to implement the flows, care needs to be taken to ensure the stability of the proposed estimator and its performance on smaller data sets. As seen in the experiments, the method performs particularly well in cases where the random quantities are based on Gaussian distribution. In future work we aim to explore the possibilities of using different classes of base distributions. This includes the possibility of dealing with discrete distributions, a situation that is handled well by critic-based methods~\citep{mine,Infonce}.
Another direction involves evaluating our method in view of downstream applications that require the computation of mutual information and comparing its performance in these settings with other generative approaches that were recently introduced~\citep{minde,duong2023diffeomorphic,butakov2024mutual}. In particular, in light of recent work~\citep{kong2023interpretable, minde}, it would be interesting to explore multimodal examples, such as the MI between image data and text embeddings. 

\begin{ack}
The first author was funded by a China Scholarship Council scholarship.
\end{ack}



\bibliography{mybibfile}
\clearpage
\newpage

\appendix

\section{Real NVP}
An alternative architecture for implementing our approach is based on Real NVP. Real NVP, proposed by \citet{RealNVP}, is a class of normalizing flows constructed using simple and flexible bijections with efficient computation of the Jacobian determinant. Each transformation in Real NVP is referred to as an affine coupling layer. Given a $d$-dimensional input $x$ and a partition point $d_m < d$, the output $y$ of an affine coupling layer is defined by the following equations:
\begin{equation*}
        y_{1:d_m} = x_{1:d_m},\ y_{d_m+1:d} = x_{d_m+1:d} \odot \exp(s(x_{1:d_m})) + t(x_{1:d_m}),
\end{equation*}
where $s$ and $t$ represent the scale and translation functions, mapping $\R^{d_m} \to \R^{d-d_m}$, and $\odot$ denotes the Hadamard (element-wise) product. The Jacobian matrix of this transformation is given by:
\begin{equation*}
    \begin{bmatrix}
        \mathbb{I}_{d_m} & 0 \\
        \frac{\partial y_{d_m+1:d}}{\partial x_{1:d_m}^T} & \textnormal{diag}(\exp(s(x_{1:d_m})))
    \end{bmatrix},
\end{equation*}
where $\textnormal{diag}(\exp(s(x_{1:d_m})))$ is a diagonal matrix whose diagonal elements correspond to the vector $\exp(s(x_{1:d_m}))$. As the Jacobian is triangular, its determinant can be efficiently computed as $\exp(\sum_j s(x_{1:d_m})_j)$. Additionally, since the computation does not depend on the Jacobian of $s$ and $t$, arbitrarily complex functions can be used for $s$ and $t$. A common choice is deep convolutional neural networks with more features in the hidden layers than in the input and output layers.
\begin{example}\label{ex2:nf}
Consider a function $f\colon \R^4\to \R^4$, given as a composition $f=f^2\circ f^1\circ f^0$ with $f^0,f^1,f^2$. Hence, the corresponding neural network has the form shown in Figure~\ref{fig2:nn}.
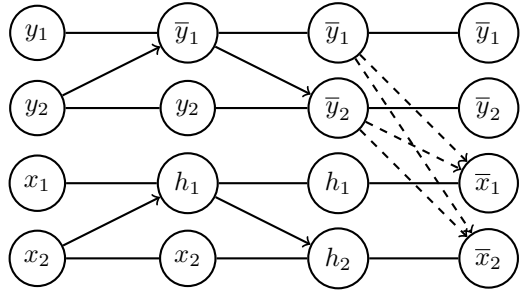
\begin{figure}[!ht]
\centering
\begin{tikzpicture}[thick, scale=1.0]
  \node [every neuron] (I1) at (-4,1.5) {$y_1$};
  \node [every neuron] (I2) at (-4,0.5) {$y_2$};
  \node [every neuron] (I3) at (-4,-0.5) {$x_1$};
  \node [every neuron] (I4) at (-4,-1.5) {$x_2$};
  \node [every neuron] (M1) at (-2,1.5) {$\overline{y}_1$};
  \node [every neuron] (M2) at (-2,0.5) {$y_2$};
  \node [every neuron] (M3) at (-2,-0.5) {$h_1$};
  \node [every neuron] (M4) at (-2,-1.5) {$x_2$};
  \node [every neuron] (N1) at (0,1.5) {$\overline{y}_1$};
  \node [every neuron] (N2) at (0,0.5) {$\overline{y}_2$};
  \node [every neuron] (N3) at (0,-0.5) {$h_1$};
  \node [every neuron] (N4) at (0,-1.5) {$h_2$};
  \node [every neuron] (O1) at (2,1.5) {$\overline{y}_1$};
  \node [every neuron] (O2) at (2,0.5) {$\overline{y}_2$};
  \node [every neuron] (O3) at (2,-0.5) {$\overline{x}_1$};
  \node [every neuron] (O4) at (2,-1.5) {$\overline{x}_2$};
  \draw[-] (I1) -- (M1);
  \draw[-] (I2) -- (M2);
  \draw[-] (I3) -- (M3);
  \draw[-] (I4) -- (M4);  
  \draw[->] (I2) -- (M1);
  \draw[->] (I4) -- (M3);
  \draw[-] (M1) -- (N1);
  \draw[-] (M2) -- (N2);
  \draw[-] (M3) -- (N3);
  \draw[-] (M4) -- (N4);  
  \draw[->] (M1) -- (N2);
  \draw[->] (M3) -- (N4);
  \draw[-] (N1) -- (O1);
  \draw[-] (N2) -- (O2);
  \draw[-] (N3) -- (O3);
  \draw[-] (N4) -- (O4);  
  \draw[->,dashed] (N1) -- (O3);
  \draw[->,dashed] (N2) -- (O3);
  \draw[->,dashed] (N1) -- (O4);
  \draw[->,dashed] (N2) -- (O4);
\end{tikzpicture}
 \hspace{5mm}
\caption{A Real NVP Flow $f(y,x)$. Solid lines represent identical units. Arrows represents the affine coupling transformations.}\label{fig2:nn}
 \hspace{150mm}
\end{figure}
\end{example}
\section{Theoretical results}
The flows we consider to implement our MI estimator are based on the factorization
\begin{equation*}
  p(x,y) = p(x \ | \ y) \cdot p(y),
\end{equation*}
with $x\in \R^n$, $y\in \R^n$. The corresponding flows are implemented by block-triangular maps. 
We call $T\colon \R^n\to \R^n$ triangular if the following conditions are satisfied for $i\in \{1,\dots,n\}$:

\begin{enumerate}
\item $T_i(x)=g_i(x_1,\dots,x_i)$ (in particular, $\partial T_i/\partial x_j=0$ for $j>i$);
\item For any $x_{1:i-1}=(x_1,\dots,x_{i-1})$, $T_i(x_{1:i-1},x_i)$ is monotonically increasing as a function of $x_i$.
\end{enumerate}

It is easy to see that triangular maps are invertible. It is a well-known result that normalizing flows can be implemented as triangular maps~\citep{triangular}. In our setting, we are interested in block-triangular maps of a specific form. We call a map $T\colon \R^{2n}\to \R^{2n}$ block triangular, if $T=(T_1(y),T_2(y,x))$, where each $T_i\colon \R^n\to \R^n$ and $\det \mathrm{d}T_1(y)>0$, as well as $\det \mathrm{d}_yT_2(y,x)>0$ for each $y$ and $x$. While in our case we restrict to the case where $x$ and $y$ have the same dimension, everything that follows generalizes to different dimensions. Recall that if $T\colon \R^{2n}\to \R^{2n}$ is a map and $\mu$ is a probability measure on $\R^{2n}$, then the push-forward measure is given by
\begin{equation*}
  T_*\mu(A) = \mu(T^{-1}(A)).
\end{equation*}
Moreover, if $\mu$ has density $p(x,y)$ and we denote by $q$ the density of the push-forward measure, then
\begin{equation*}
  q(T(y,x))\cdot |\det \mathrm{d}T(y,x)| = p(y,x).
\end{equation*}
In particular, if $T$ is a normalizing flow, then $q$ is a Gaussian density. In particular, if $T$ is block-triangular as above, then 
\begin{equation*}
|\det \mathrm{d}T(y,x)|=|\det \mathrm{d}_yT_1(y)|\cdot  |\det \mathrm{d}_xT_2(y,x)|
\end{equation*}
and the normalizing flow corresponds to the factorization
\begin{align*}
  p(y) &= q(T_1(y)) \cdot |\det \mathrm{d}_yT_1(y)|\\
  p(x \ | \ y) &= q(T_2(y,x)) \cdot |\det \mathrm{d}_x T_2(y,x)|,
\end{align*}
where again $q$ is a Gaussian density on $\R^n$. It is not a priori clear how this approach leads to a normalizing flow for $p(x)$, say
\begin{equation*}
  p(x) = q(T_x(x)) \cdot |\det \mathrm{d}_xT_x(x)|,
\end{equation*}
in such a way that $T_2(y,x)$ can be built from this. The following theorem answers this question.

\begin{theorem}
Let $\mu,\nu$ be absolutely continuous probability measures on $\R^{2n}$ and let $\mu_1,\mu_2,\nu_1,\nu_2$ denote the marginal distributions on the first and last $n$ coordinate, respectively. Let $T_1\colon \R^n\to \R^n$, $T_2\colon \R^n\to \R^n$ be maps such that
\begin{equation*}
  (T_1)_*\mu_1=\nu_1, \quad (T_2)_*\mu_2=\nu_2.
\end{equation*}
Then there exists a map $T\colon \R^{2n}\to \R^{2n}$ such that $T_*\mu=\nu$ and
\begin{equation*}
  T(y,x) = (T_1(y),\overline{T}(y,T_2(x)))
\end{equation*}
for a suitable map $\overline{T}\colon \R^{2n}\to \R^n$. 
\end{theorem}

We note that the existence of $T_1$ and $T_2$ as in the theorem follows from~\citet{triangular} (they may even be chosen to be triangular maps).

\begin{proof}
 Consider the map
\begin{equation*}
  U\colon \R^{2n}\to \R^{2n}, \quad U(y,x) = (y,T_2(x)).
\end{equation*}
Let $\overline{\mu}=U_*\mu$. By~\citet{triangular}, we know that there exists a triangular transformation $V=(T_1,\overline{T})\colon \R^{2n}\to\R^{2n}$, with $\overline{T}\colon \R^n\to \R^n$ and $V_*\overline{\mu}=\nu$. It follows that 
\begin{equation*}
  \nu = V_*\overline{\mu}=V_*U_*\mu = (V\circ U)_*\mu.
\end{equation*} 
Since $V\circ U=T$ as stated in the theorem, the claim follows.
\end{proof}

\begin{corollary}\label{cor:joint}
Let $p$ be the joint density of $(X,Y)$ and consider the factorization
\begin{equation*}
  p(x,y) = p_Y(y) \cdot p_{X|Y}(x|y).
\end{equation*}
Let $f_{(x)}\colon \R^n\to \R^n$ be a normalizing flow for the marginal density $p_X(x)$. 
There exists a block-triangular map $f=(f_1,f_2)$ that is a normalizing flow for $p(x,y)$, 
where $f_2(y,x)=\overline{f}(y,f_{(x)}(x))$ for a map $\overline{f}\colon \R^{2n}\to \R^n$, and such that
\begin{equation*}
  p(x|y) = q(\overline{f}(y, f_{(x)}(x))) \cdot |\det \mathrm{d}_x f_2(y,x)|,
\end{equation*}
where $q$ is a Gaussian density.
\end{corollary}

The idea is to implement such a flow using neural networks, in such a way that by deactivating
a certain part of the neural network for $f_2(y,x) = \overline{f}(y,f_{(x)}(x))$, we get the flow $f_{(x)}$. 
In analogy to Block Neural Autoregressive Flows, we now consider 

\begin{align*}
  f\colon \R^{2n}\to \R^{2n}, \quad f(y,x) &= (f_1(y,x), f_2(y,x))^T\\
  & = f^{\ell}\circ \cdots \circ f^{1}, 
\end{align*}

where $f_1\colon \R^{2n}\to \R^n$ and $f_2\colon \R^{2n}\to \R^n$. We assume that each $f^k$, $1\leq k<\ell$, is of the form

\begin{equation*}
  f^k({y},{x}) = \sigma \left(\begin{pmatrix} g({B}^k_{11}) & {0}\\ {B}^k_{21} & g({B}^k_{22}) \end{pmatrix} \begin{pmatrix} {y} \\ {x}\end{pmatrix} + \begin{pmatrix} {b}^k_1\\ {b}^k_2\end{pmatrix}\right) 
\end{equation*} 

with ${B}^k_{ij}\in \R^{m_{k}\times m_{k-1}}$, $g=\exp$ and $m_0=d$. For $k=\ell$ we omit the activation function and set $m_k=d$. The following basic result is shown along the lines of~\citet{BNAF}. 

\begin{theorem} The Jacobian $\mathrm{d}f$ is a $2\times 2$ block triangular matrix with $d\times d$ blocks,
\begin{equation*}
  \mathrm{d}f({y},{x}) = \begin{pmatrix} \mathrm{d}_{{y}}f_1({y}) & {0}\\ \mathrm{d}_{{y}}f_2({y},{x}) & \mathrm{d}_{{x}}f_2({y},{x}) \end{pmatrix}.
\end{equation*}
Moreover, if $\sigma$ is strictly increasing, then $\det \mathrm{d}_{{y}}f_1({y})>0$ and $\det \mathrm{d}_{{x}}f_2({y},{x}) > 0$ for all $({y},{x})$. 
\end{theorem}

One consequence of this characterization of block-triangular flows is that if we train this neural network with the cost function
\begin{equation*}
  \log q(f_2(y,x)) + \det \mathrm{d}_xf_2(y,x),
\end{equation*}
we can obtain the entropy associated to the marginal density, $H(X)$, by ``deactivating'' the weights that operate on $y$. Corollary~\ref{cor:joint} suggests that given enough expressive power of our neural network architecture, we can train the network to both approximate $H(X|Y)$ and $H(X)$ by first training the marginal density $f_{(x)}$ and then training for $(f_1,\overline{f})$
on samples $(y_i,f_{(x)}(x_i))$. Along the lines of the appendix in~\citet{BNAF}, one can show that any conditional density can be approximated by a block-triangular flow as described.


\section{Additional Experiments}\label{additional_exp_NDoE}

\subsection{MI Estimation with Nonlinear Transformations} \label{4.1}
The first experiment we conducted are focused on estimating MI with samples that are generated with nonlinearly transformed Gaussians. Here we consider asinh and wiggly transformations that are provided in \citep{czyz2023beyond}. The different MI estimators were trained on the $20$-dimensional datasets of the varying sample size ($32$K, $64$K and $128$K). All the other settings remained the same.

\textbf{Results.} The results are presented in Figure~\ref{Gaussian_asinh} and Figure~\ref{Gaussian_wiggly}. Overall, all discriminative methods tend to underestimate MI, while our proposed methods demonstrated superior performance. In particular, \textbf{NDoE, BNAF} exhibited less bias compared to \textbf{BNAF} in cases with additional cubic transformations. Although \textbf{NDoE, Real NVP} performed worse than both \textbf{NDoE, BNAF} and \textbf{BNAF}, it still exhibited less bias than all other discriminative methods and the \textbf{DoE} estimators across all scenarios.

\begin{figure*}[!ht]
\centering
    \begin{subfigure}[b]{1.0\linewidth}
    \centering
    \includegraphics[scale=0.59]{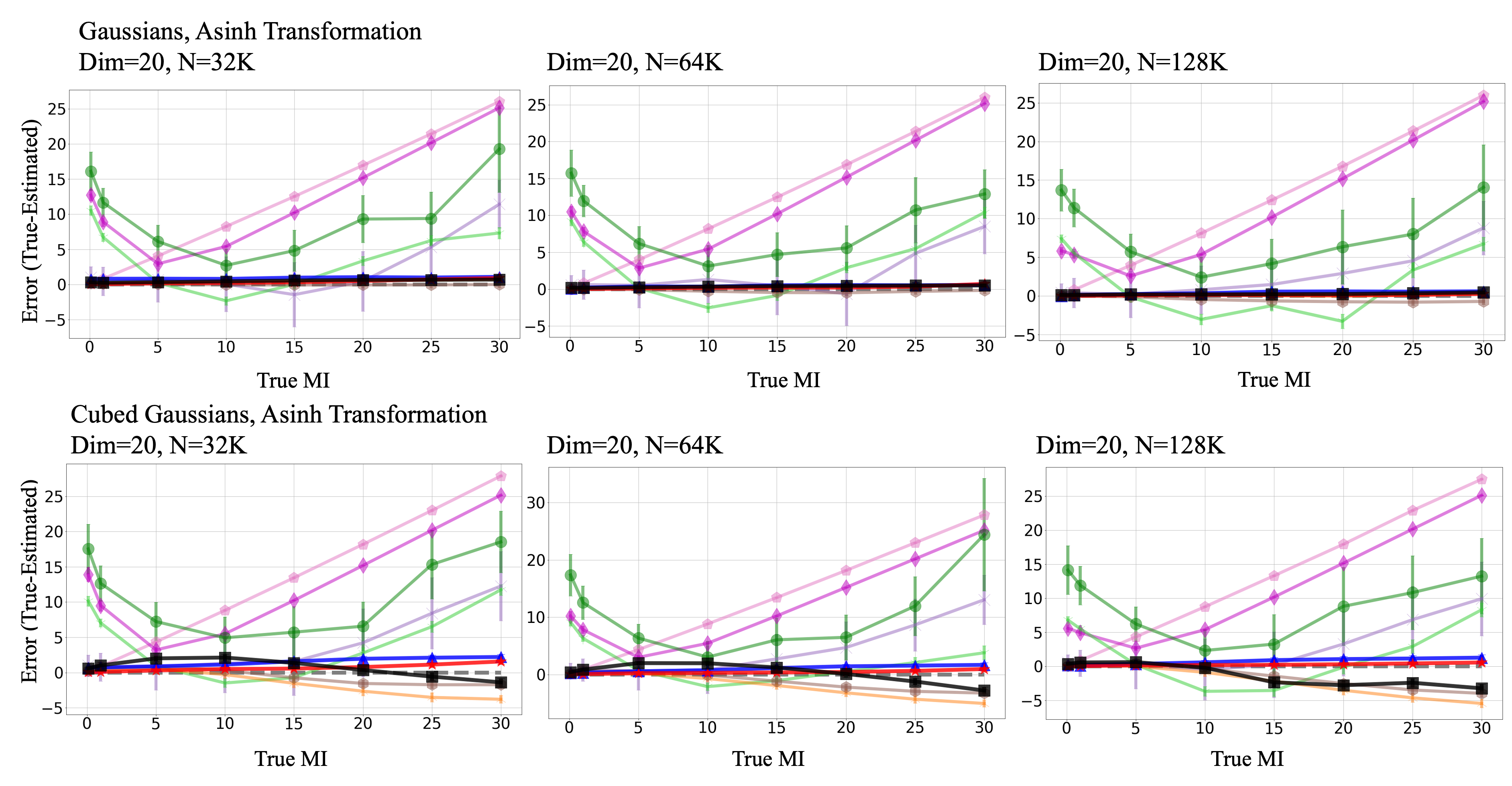}
    \end{subfigure}\\

    \begin{subfigure}[b]{1.0\linewidth}
    \centering
    \includegraphics[scale=0.65]{plots/legend_Gaussian.png}
    \end{subfigure}
\caption{MI estimation between asinh-transformed Gaussian variables (Top) and between asinh-transformed Gaussian variables with a cubic transformation (Bottom). The estimation error $(I(x,y)-\hat{I}(x,y))$ are reported. Closer to zero is better.}
\label{Gaussian_asinh}
 \hspace{150mm}
\end{figure*}   
\begin{figure*}[!ht]
\centering
    \begin{subfigure}[b]{1.0\linewidth}
    \centering
    \includegraphics[scale=0.61]{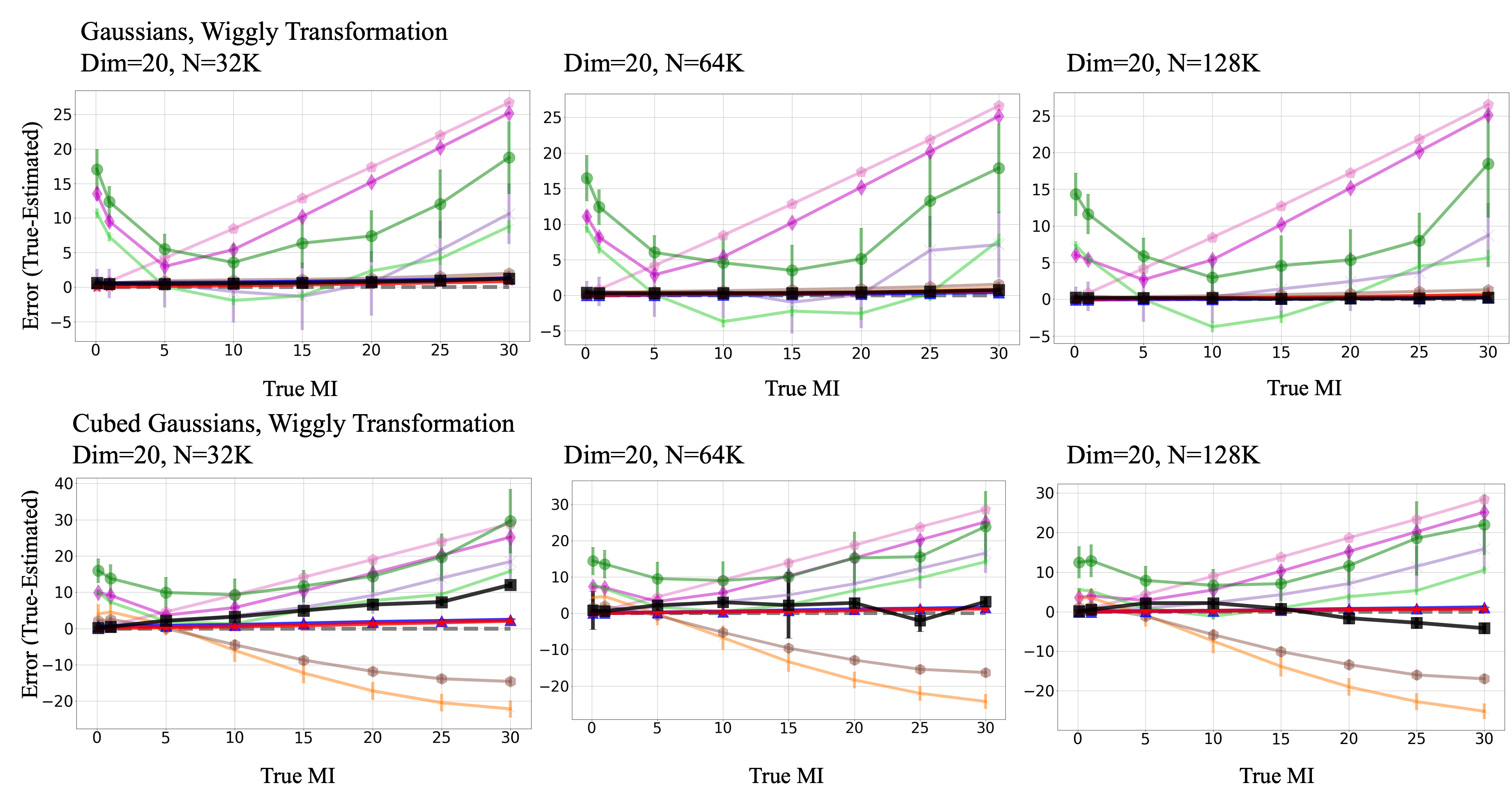}
    \end{subfigure}\\

    \begin{subfigure}[b]{1.0\linewidth}
    \centering
    \includegraphics[scale=0.65]{plots/legend_Gaussian.png}
    \end{subfigure}
\caption{MI estimation between wiggly-transformed Gaussian variables (Top) and between wiggly-transformed Gaussian variables with a cubic transformation (Bottom). The estimation error $(I(x,y)-\hat{I}(x,y))$ are reported. Closer to zero is better.}
 \hspace{150mm}
\label{Gaussian_wiggly}
\end{figure*}   

\subsection{MI Estimation on Extremely Small-sized Samples} \label{4.2}
The second experiment we conducted is similar to the last experiment in Section \ref{4.1}. We trained the different estimators on the training set of size $1024$ and tested them on another independently generated $1024$ samples to obtain MI estimates. All the other settings remained the same. We repeated the training process for $20, 50, 100$ and $200$ epochs. 

\textbf{Results.} The results of 20 dimensionalities are presented in Figure~\ref{Gaussian_N1024}. With the small number of epochs of training on extremely small-sized samples, all methods gave a bad performance on large MI. However, \textbf{DEMI} and \textbf{NDoE, BNAF} still obtained relatively good estimates when the true MI is close to zero. With the increase of the number of training epochs, the estimates of \textbf{NDoE, BNAF} started to converge to the true MI, while other discriminative methods lead to large errors. We noticed that \textbf{BNAF} also shows the trend of convergence, but the estimation results are worse than \textbf{NDoE, BNAF}. \textbf{NDoE, Real NVP} failed to achieve realistic results in this experiment.

\begin{figure*}[!ht]
\centering
    \begin{subfigure}[b]{1.0\linewidth}
    \centering
    \includegraphics[scale=0.62]{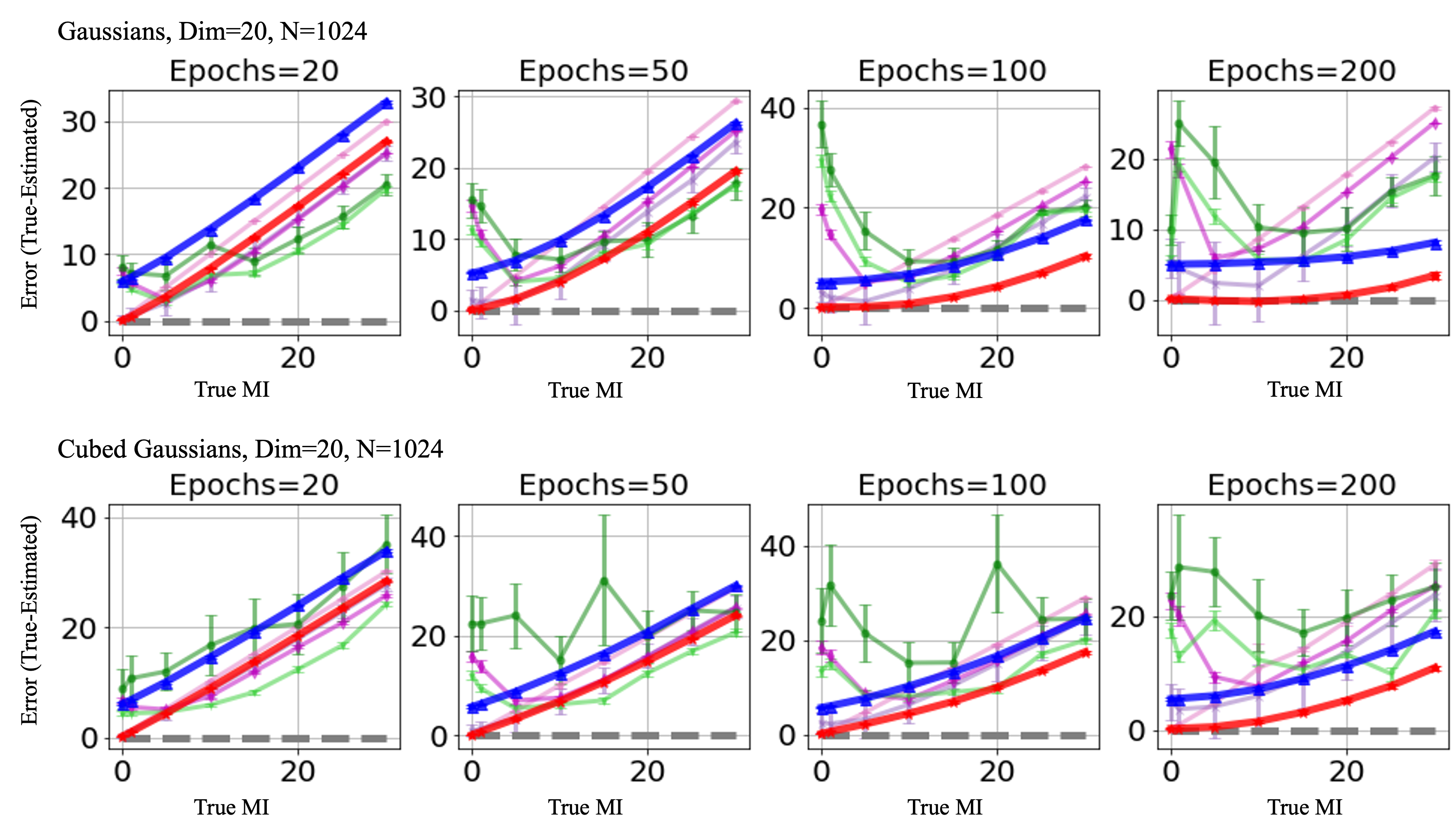}
    \end{subfigure}\\

    \begin{subfigure}[b]{1.0\linewidth}
    \centering
    \includegraphics[scale=0.65]{plots/legend_Gaussian.png}
    \end{subfigure}
\caption{MI estimation between multivariate Gaussian variables (Top) and between multivariate Gaussian variables with a cubic transformation (Bottom). The estimators are trained on training data of size 1024 with varying training epochs and the estimates are obtained from testing data of size 1024. The estimation error $(I(x,y)-\hat{I}(x,y))$ are reported. Closer to zero is better.}
 \hspace{150mm}
\label{Gaussian_N1024}
\end{figure*} 

\subsection{MI Estimation on Correlated Uniforms and Student's t Distributions}\label{appendix:b2.1}
It is well-known that non-Gaussian distributions, especially distributions with long tails, remains challenging for MI estimation for many reasons. In this experiment, we sampled from two random variables $X$ and $Y$ of correlated Uniform and Student's t distributions, for which the MI can be exactly obtained from their known correlation. The construction of Student's t distribution follows the idea of \citet{czyz2023beyond} where the Gaussians are taken from the first experiment. Note that, $I(X,Y)>0$ for the generated Student's t distribution even for independent Gaussians $X,Y$ in this example. We trained the different estimators on the training set of 128K samples and tested them on another independently generated 10240 samples to obtain MI estimates. All the other settings remained the same. 

\textbf{Results.} The results for Uniform variables are shown in Figure~\ref{Uniform_D20}, while those for Student's t variables are presented in Figure~\ref{Student_t_D20}. In the Uniform case, our proposed method provides better estimates with relatively small variance. \textbf{DEMI} achieves competitive results for small MI, but the error remains substantial for larger MI values. Notably, \textbf{NDoE, Real NVP} delivers even better performance when a cubic transformation is applied.

Unfortunately, most estimators fail to yield realistic results for Student's t distributions, whereas our method maintains small bias and variance in the estimates. \textbf{InfoNCE} displays a large bias for high MI, and \textbf{DEMI} shows very high variance. This experiment also excludes the influence of learning the target distribution from the same base distribution (Simple Gaussian) in \textbf{NDoE} and \textbf{BNAF}. Although the bias increases compared to the Gaussian examples, \textbf{NDoE} continues to demonstrate superior performance compared to other methods.
\begin{figure*}[!ht]
\centering
    \begin{subfigure}[b]{1.0\linewidth}
    \centering
    \includegraphics[scale=0.62]{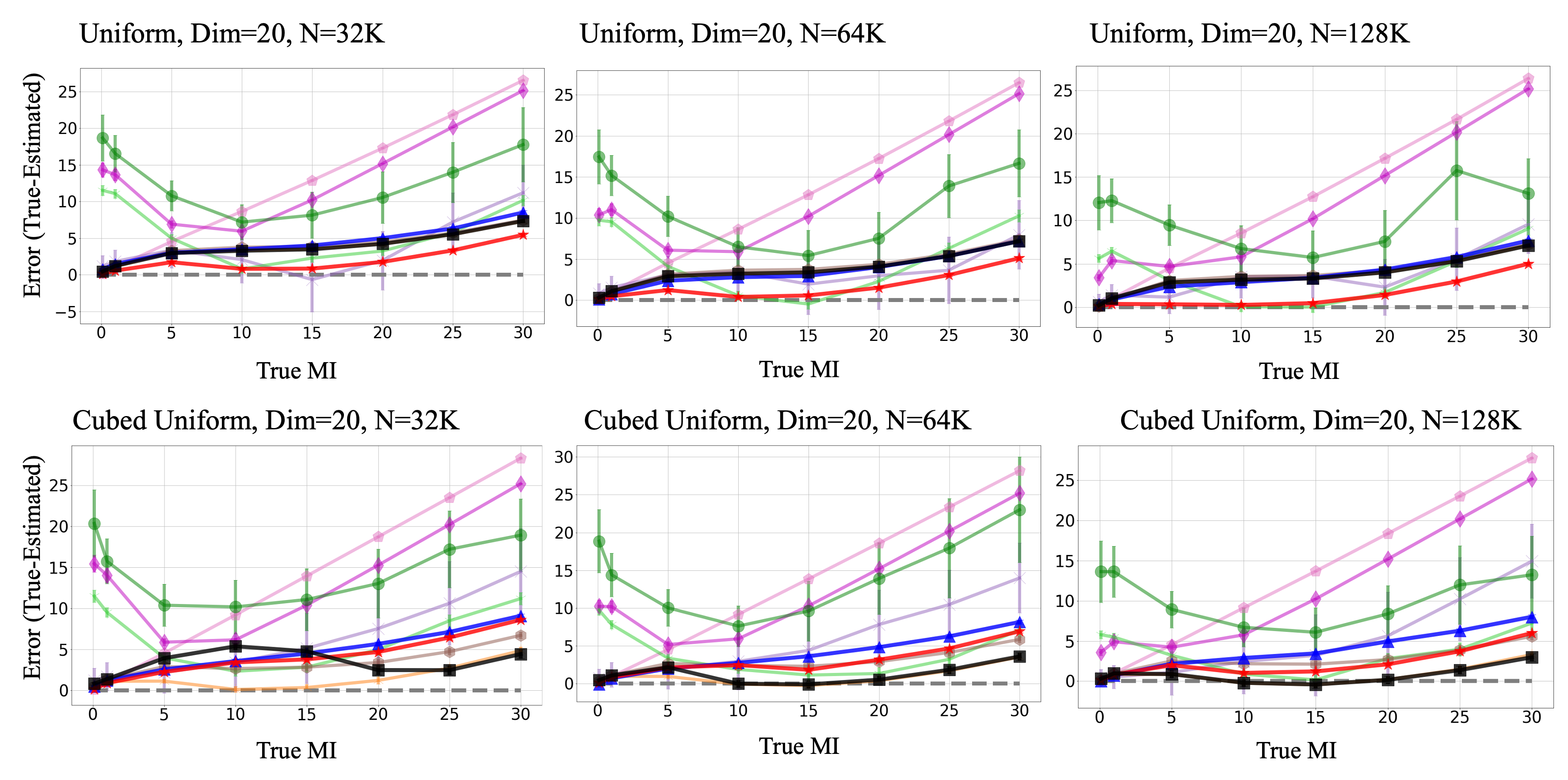}
    \end{subfigure}\\

    \begin{subfigure}[b]{1.0\linewidth}
    \centering
    \includegraphics[scale=0.65]{plots/legend_Gaussian.png}
    \end{subfigure}
\caption{MI estimation between multivariate Uniform variables (Top) and between multivariate Uniform variables with a cubic transformation (Bottom). The estimation error $(I(x,y)-\hat{I}(x,y))$ are reported. Closer to zero is better.}
 \hspace{150mm}
\label{Uniform_D20}
\end{figure*} 

\begin{figure*}[!ht]
\centering
    \begin{subfigure}[b]{1.0\linewidth}
    \centering
    \includegraphics[scale=0.58]{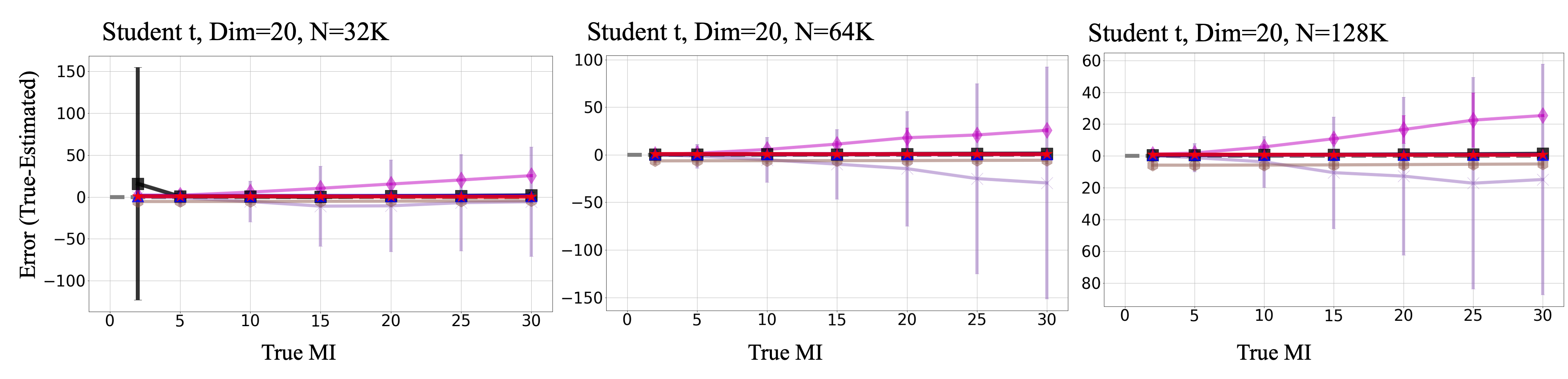}
    \end{subfigure}\\
     \begin{subfigure}[b]{1.0\linewidth}
    \centering
    \includegraphics[scale=0.65]{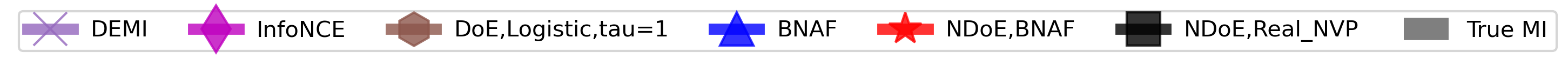}
    \end{subfigure}

\caption{MI estimation between multivariate Student's t random variables. The estimation error $(I(x,y)-\hat{I}(x,y))$ are reported. Closer to zero is better.}
 \hspace{150mm}
\label{Student_t_D20}
\end{figure*} 
\subsection{Comparison between NDoE and BNAF}
To demonstrate the empirical out-performance of our \textbf{NDoE, BNAF} method against the core baseline, \textbf{BNAF}, which utilizes two separate flows with identical hyperparameters and initializations, we consider the cubed Gaussians distribution as an example. The estimated MI is plotted against the number of training epochs to showcase the comparative performance. The results are illustrated in Figure~\ref{comparison}. All other settings remain the same.

\textbf{Results.} As shown in the plots, \textbf{NDoE, BNAF} demonstrated better convergence behavior with relatively small error after just 10 epochs of training, whereas \textbf{BNAF} required between 20 to 60 epochs to achieve competitive results. This performance gap is particularly pronounced when the number of training samples is smaller. Another noteworthy observation is that \textbf{NDoE, BNAF} empirically behaved as an upper-bound estimator, consistently producing estimates greater than zero, which aligns well with the fundamental properties of mutual information (MI). In contrast, \textbf{BNAF} does not exhibit this property, which is often regarded as a key limitation of generative methods. We attribute this phenomenon to the Correlation Boosting Effect proposed by \citet{biksg}, though we do not provide a rigorous proof at this stage.
\begin{figure*}[!ht]
\centering
    \begin{subfigure}[b]{1.0\linewidth}
    \centering
    \includegraphics[scale=0.75]{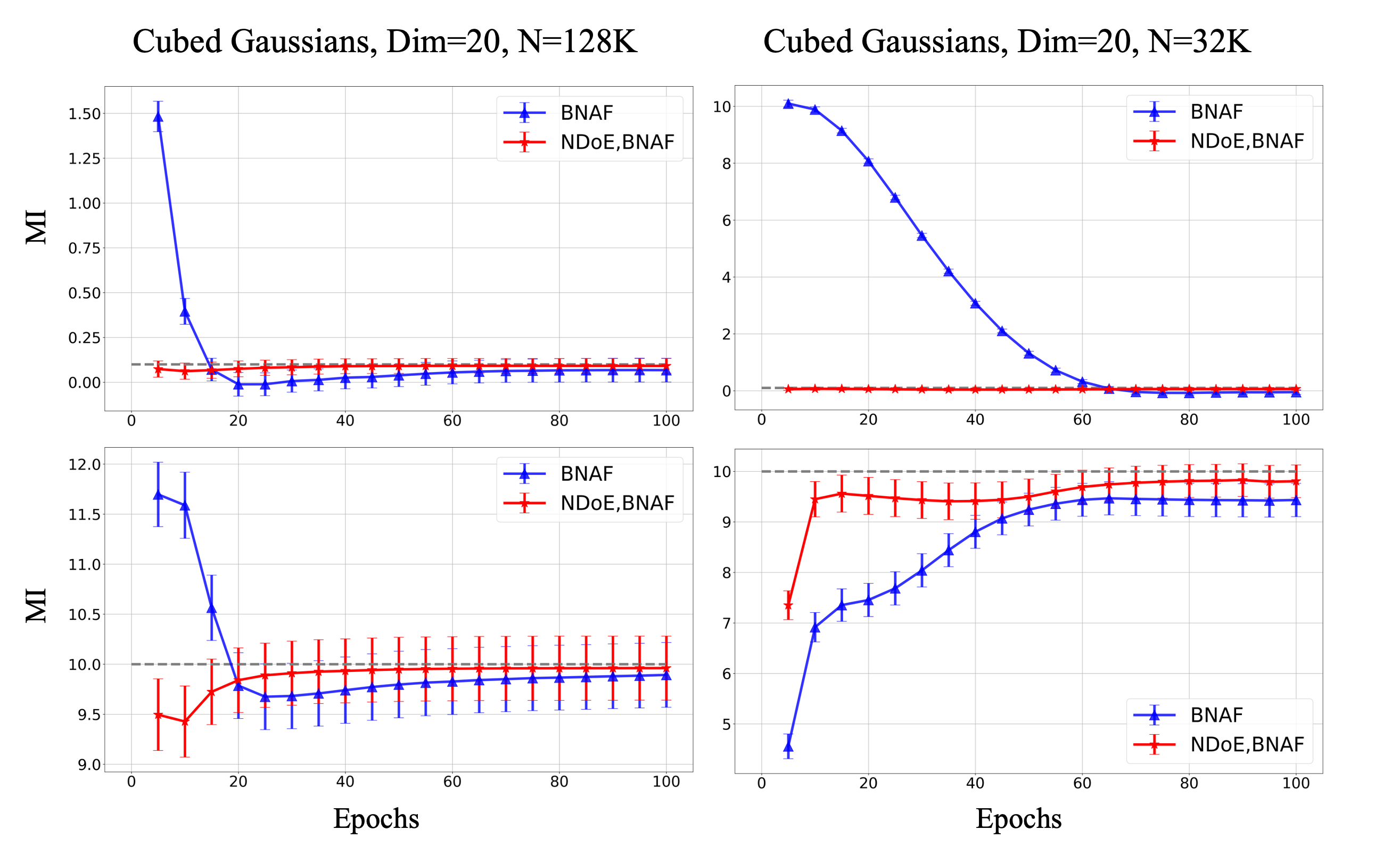}
    \end{subfigure}\\

\caption{MI estimation between multivariate cubed Gaussian variables. The estimation $\hat{I}(x,y)$ of varying training epochs versus the true underlying MI are reported.}
 \hspace{150mm}
\label{comparison}
\end{figure*} 

\subsection{Long-run Training Behavior of NDoE}\label{appendix:b2.2}
We noticed from the last three experiments that \textbf{DEMI} underestimates the MI when the random variables are highly dependent. The underestimation is not alleviated by the increased number of training epochs. At the same time, our method shows the reduction of bias with the repetition of training. Thus, we proposed an assumption that the discriminative methods diverge to the true MI for high mutual dependence, while generative models have a good convergence property with sufficient training samples, with the precondition of enough expressive power of the neural network. In other words, there is a systematic bias in the discriminative methods which are positively correlated to mutual dependence. To verify this experimentally, we conducted a long-run training behavior experiment using similar settings of \citet{demi}. Here we only verify the long-run training behavior for \textbf{NDoE, BNAF}. All the estimators were trained on the 20-dimensional Gaussian  and Cubed Gaussian case for 100000 training steps with a batch size of $128$. Samples were drawn directly from the generating distributions. We did this for four ground-truth MI values of $0.1,10,20$, and $30$. 

\textbf{Results.} The results are shown in Figure~\ref{longrun} and Figure~\ref{longrun_cubed}. Among all the methods, our method shows the best convergence behaviour for all MI with the increase of training epochs, and the variances remains relatively small \textbf{DEMI} method is competitive when true MI is close to $0$. However, it underestimates MI vastly for large true MI. We also noticed that \textbf{SMILE} with the parameter $\tau=1, 5$ tends to overestimate MI even though they are based on a lower bound of it. \citet{limit_mi} suggests that the reasoning behind this could be the sensitivity of the estimate of $-\textnormal{ln}\ \mathbb{E}[e^{f(x,y)}]$ to outliers. However, the parameter $\tau$ actually clips the term $e^{f(x,y)}$ to the interval $[e^{-\tau},e^{\tau}]$, which removes outliers from the neural network outputs.

\begin{figure*}[!ht]
\centering
    \begin{subfigure}[b]{1.0\linewidth}
    \centering
    \includegraphics[scale=0.75]{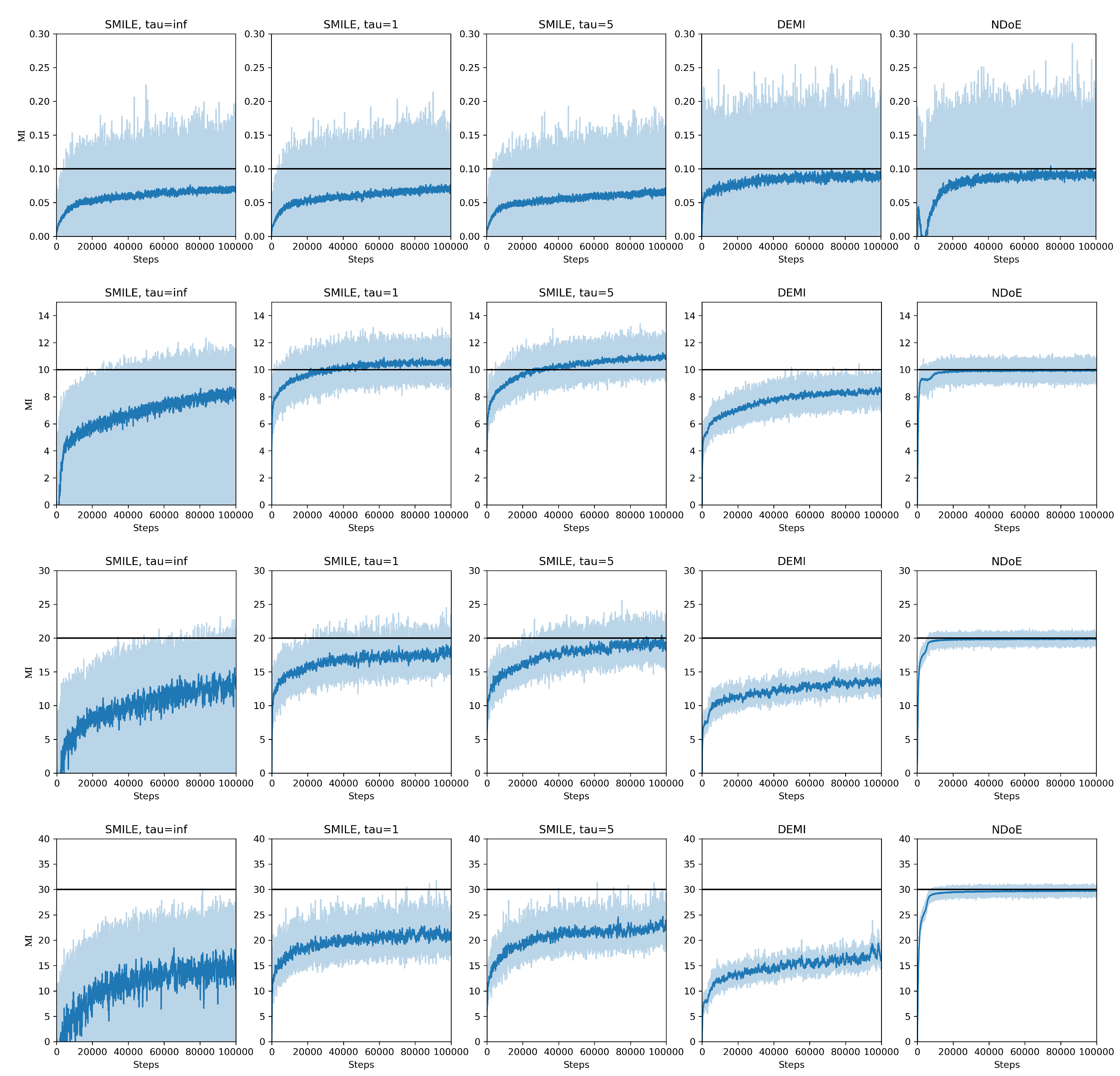}
    \end{subfigure}\\

\caption{MI estimation between multivariate Gaussian variables. The estimation of each training step $\hat{I}(x,y)$ versus the true underlying MI are reported. The estimators are trained on 100000 training steps with varying true MI.}
\label{longrun}
 \hspace{150mm}
\end{figure*} 

\begin{figure*}[!ht]
\centering
    \begin{subfigure}[b]{1.0\linewidth}
    \centering
    \includegraphics[scale=0.75]{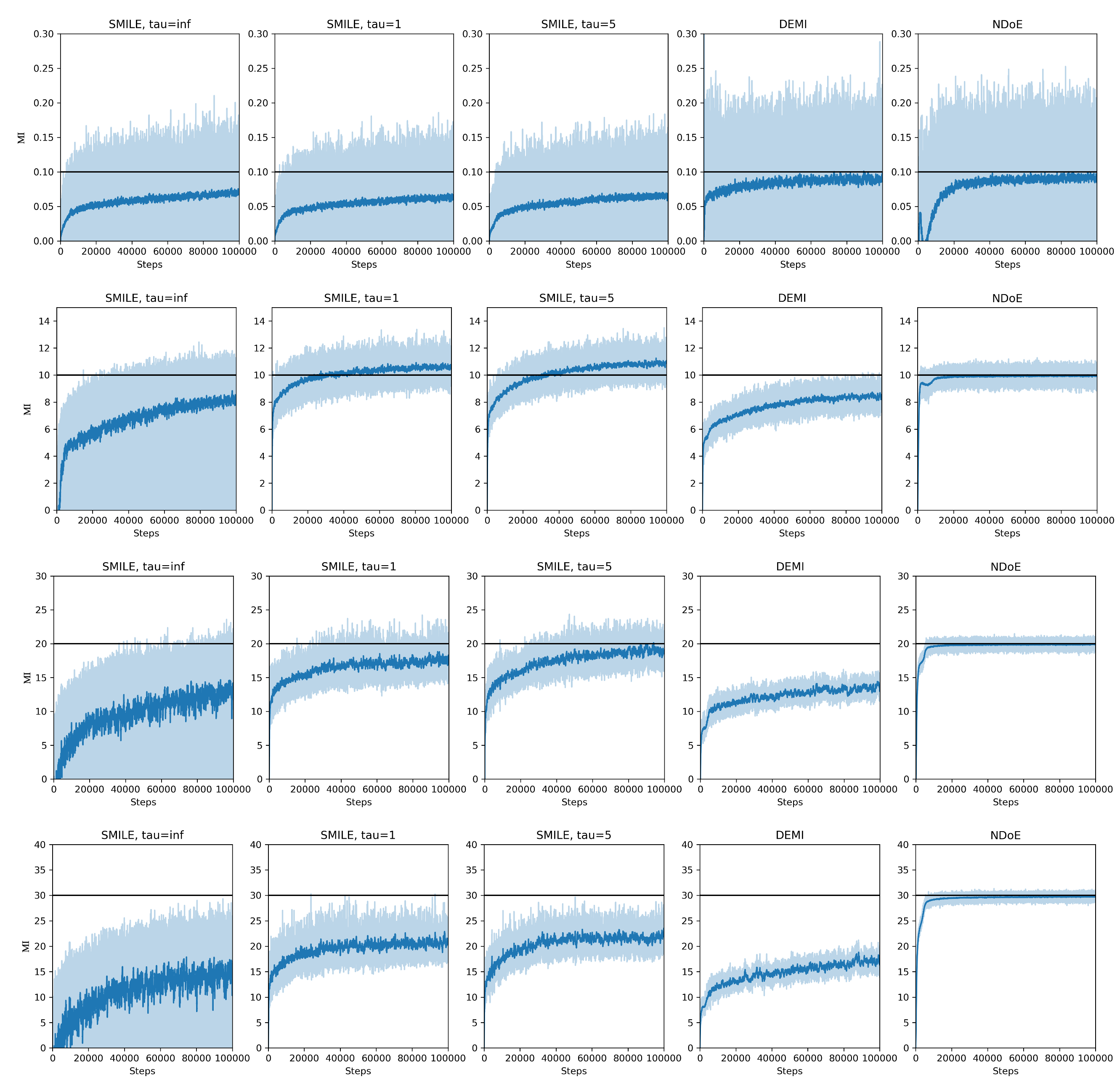}
    \end{subfigure}\\

\caption{MI estimation between multivariate Gaussian variables with a cubic transformation. The estimation of each training step $\hat{I}(x,y)$ versus the true underlying MI are reported. The estimators are trained on 100000 training steps with varying true MI..}
\label{longrun_cubed}
 \hspace{150mm}
\end{figure*}

\subsection{Asymmetry test and Varying Batchsize}
It is obvious to see that our proposed method is asymmetric, i.e. estimating the MI by using $I(X,Y)=H(X)-H(X,Y)$ and $I(X,Y)=H(Y)-H(Y,X)$ could obtain different results. We did an extra experiment on the above distribution to show that the difference is minor. The results are shown in Figure~\ref{asymm_20d}. Since our method tends to underestimate the true MI in the experiments, choosing the larger estimation will lead to less bias in most cases. Another experiment focuses on varying batchsize. It is believed that the poor performance of discriminative methods on high MI estimation is dependent on the batch size. We choose the batchsize of $64,128,256$ and $512$ to see the performance of our method and compare them with the baselines. The results are shown in Figure~\ref{batchsize_gaussian}.

\begin{figure*}[!ht]
\centering
    \begin{subfigure}[b]{1.0\linewidth}
    \centering
    \includegraphics[scale=0.20]{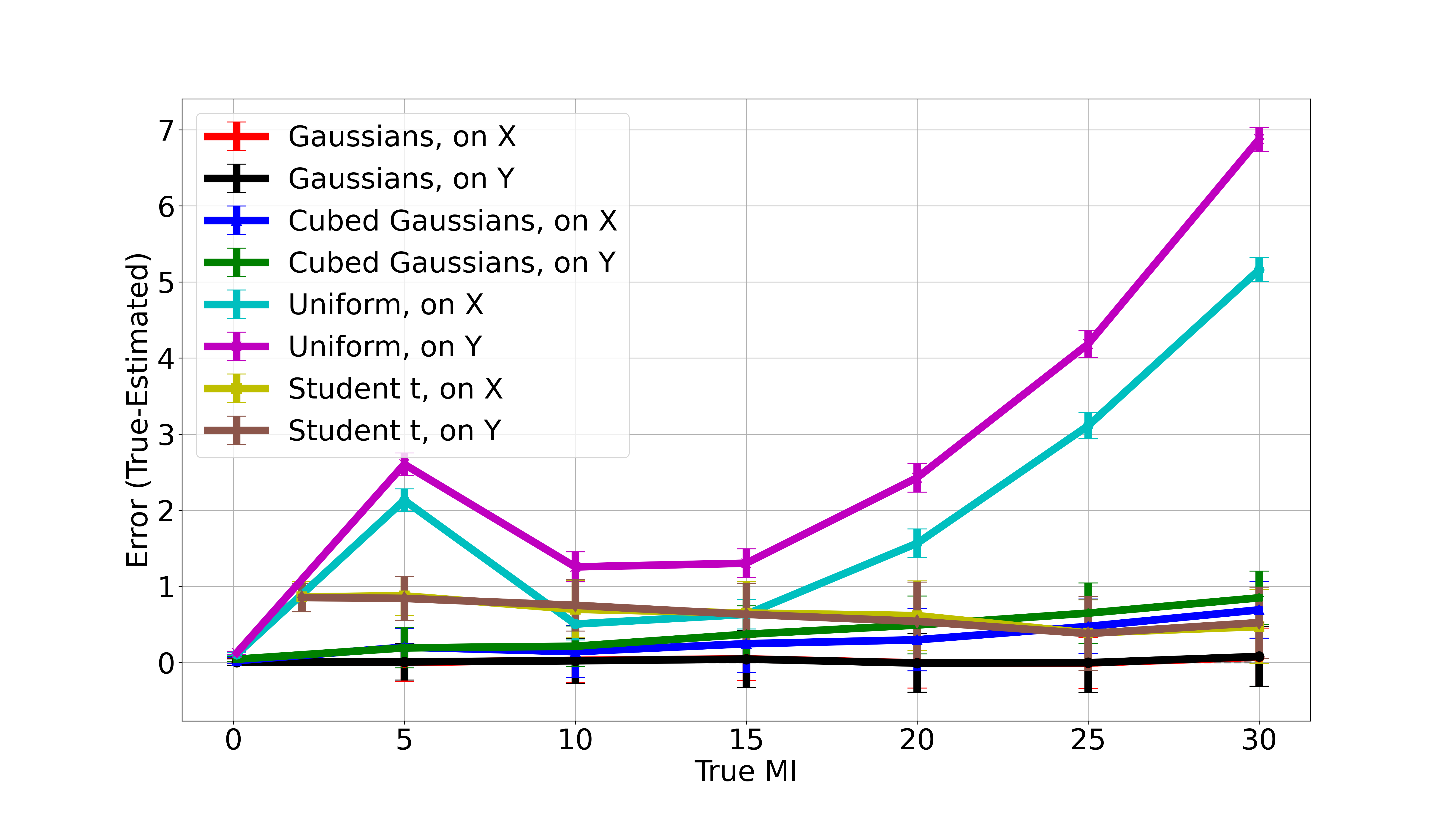}
    \end{subfigure}\\

\caption{MI estimation between random variables of different distributions. The estimation error $(I(x,y)-\hat{I}(x,y))$ are reported. Closer to zero is better.}
\label{asymm_20d}
 \hspace{150mm}
\end{figure*} 

\begin{figure*}[!ht]
\centering
    \begin{subfigure}[b]{1.0\linewidth}
    \centering
    \includegraphics[scale=0.60]{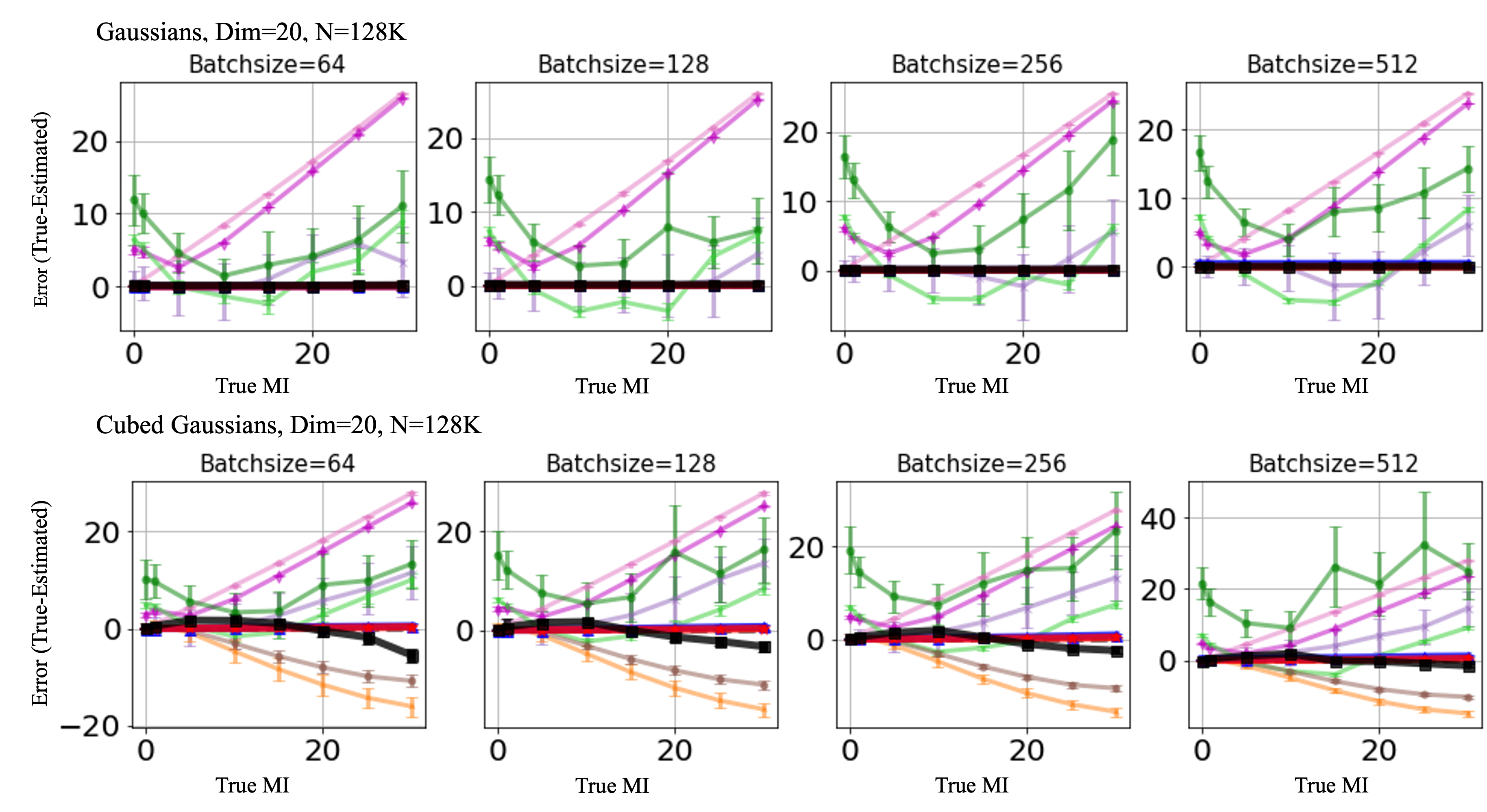}
    \end{subfigure}\\
    \begin{subfigure}[b]{1.0\linewidth}
    \centering
    \includegraphics[scale=0.65]{plots/legend_Gaussian.png}
    \end{subfigure}
\caption{MI estimation between multivariate Gaussian variables (Top) and between multivariate Gaussian variables with a cubic transformation (Bottom). The estimators are trained with varying training batchsize.  The estimation error $(I(x,y)-\hat{I}(x,y))$ are reported. Closer to zero is better.}
 \hspace{150mm}
\label{batchsize_gaussian}
\end{figure*} 

\section{Self-consistency}
In applications with real data, obtaining the ground truth MI is challenging or not possible. However, as suggested by~\citet{smine}, one can still test whether a MI estimator satisfies some of the fundamental properties of MI: $I(X,Y)=0$ if $X$ and $Y$ are independent, the data processing inequality is satisfied (that is, transforming $X$ and $Y$ should not increase the MI), and additivity. 

Following~\citet{smine}, we conducted self-consistency tests on high-dimensional images (MNIST) under three settings, where obtaining the ground truth MI is challenging. These settings involve processing images $X$ and $Y$ in different ways to assess the performance of various methods: \textbf{DEMI}, \textbf{InfoNCE}, \textbf{SMILE}, \textbf{NDoE, BNAF} and \textbf{BNAF}, where \textbf{NDoE, BNAF} and \textbf{BNAF} apply autoencoders (AE) for dimensionality reduction. \textbf{DoE} is not included as it is considered for certain failure in the experiments before. The three settings include:

\begin{enumerate}
\item $X$ is an image, and $Y$ is the same image with the bottom rows masked, leaving the top $t$ rows. The goal is to observe whether MI is non-decreasing with $t$.
Methods are evaluated under various $t$ values, normalized by the estimated MI between $X$ and itself.
\item Data-Processing. $X$ corresponds to two identical images, and $Y$ comprises the top $t_1$ and $t_2$ rows of the two images ($t_1 \geq t_2$).
The evaluation involves comparing the estimated MI ratio between $[X, X]$ and $[Y, h(Y)]$ to the true MI between $X$ and $Y$, where $h(Y)$ use $t_2=t_{1}-3$ rows.
\item Additivity. $X$ corresponds to two independent images, and $Y$ includes the top t rows of both. The assessment focuses on the estimated MI ratio between $[X_1, X_2]$ and $[Y_1, Y_2]$ relative to the true MI between $X$ and $Y$.
\end{enumerate}

{\bf Results.} The results are shown in Figures~\ref{self_consist}. Regarding the baseline, most methods correctly predict zero MI when $X$ and $Y$ are independent, thereby passing the initial self-consistency test. Additionally, the estimated MI shows a non-decreasing trend with increasing $t$, although the slopes differ among the methods. The ratio obtained by \textbf{NDoE, BNAF} is very close to the true ratio.

For the data-processing test, we set $t_2=t_1-3$. Ideally, the estimator should satisfy $\hat{I}([X, X]; [Y, h(Y )])/ \hat{I}(X, Y)\approx 1$. This is because additional processing should not result in an increase in information. All methods perform relatively well except for \textbf{NDoE, BNAF} and \textbf{BNAF}. This is possibly due to the limited capacity of AE.

In the additivity setting, the estimator should ideally double its value compared to the baseline with the same t, i.e. $\hat{I}([X_1, X_2]; [Y_1, Y_2])/ \hat{I}(X, Y )\approx 2$. Discriminative approaches did not perform well in this case, except when t was very small. As t increased, this ratio converged to 1, possibly due to initialization and saturation of the training objective. However, \textbf{NDoE, BNAF} performed well on this test except when t is small ($t=0, 3$). Compared with the results from \citet{smine}, it is promising to see improved performance by using VAE instead.

\begin{figure*}[!ht]
\centering

    \begin{subfigure}[b]{1.0\linewidth}
         \centering
         \includegraphics[scale=0.58
         ]{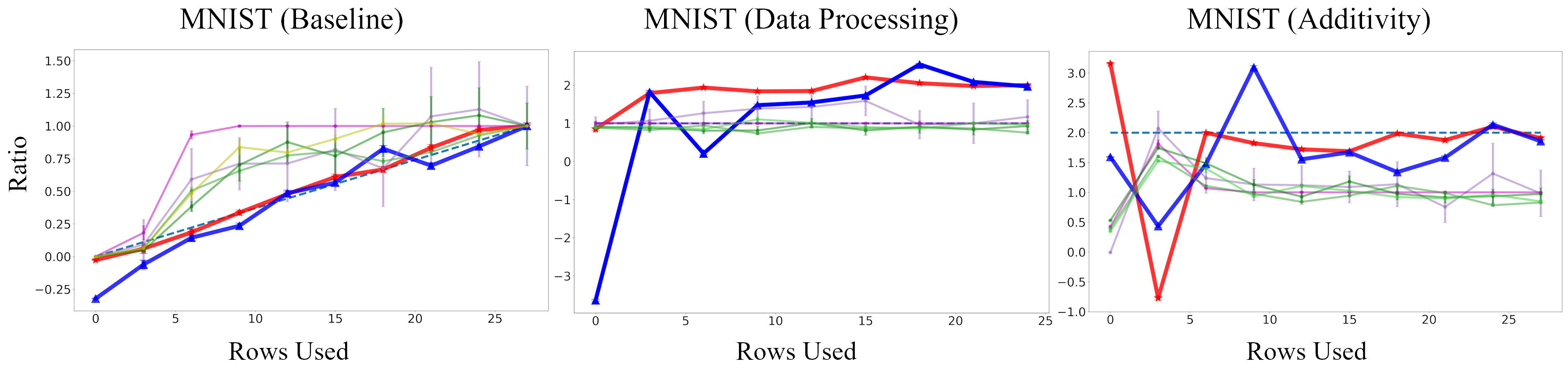}
       
    \end{subfigure}\\
    \begin{subfigure}[b]{1.0\linewidth}
    \centering
    \begin{tikzpicture}
    \node (img1)  {\includegraphics[scale=0.63]{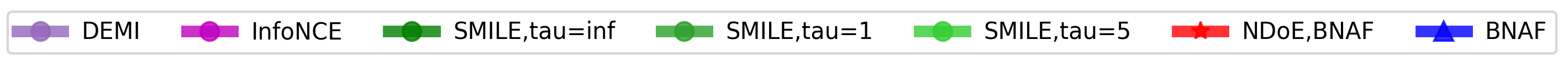}};
    \node[left=of img1, node distance=0cm,yshift=0cm,  xshift=1cm, font=\small\color{black}] {Estimator Legend:};
    \end{tikzpicture}
    \label{subfig:legend3}
    \end{subfigure}
    \caption{Evaluation on high-dimensional images (MNIST) under three settings. From left to right: Evaluation of $\hat{I}(X;Y)/\hat{I}(X;X)$; Evaluation of $\hat{I}([X,X];[Y,h(Y)]/\hat{I}(X;Y)$, where the ideal value is 1; Evaluation of $\hat{I}([X_1,X_2];[Y_1,Y_2]/\hat{I}(X;Y)$, where the ideal value is 1. $X$ is an image, $Y$ contains the top $t$ rows of $X$ and $h(Y)$ contains the top $(t-3)$ rows of $X$.}
    \label{self_consist}
     \hspace{150mm}
\end{figure*} 

\end{document}